\documentclass{article}

\usepackage[preprint]{neurips_2026}

\usepackage[utf8]{inputenc} 
\usepackage[T1]{fontenc}    
\usepackage{hyperref}       
\usepackage{url}            
\usepackage{microtype}      
\usepackage{graphicx}
\usepackage{subcaption}
\usepackage{booktabs}       
\usepackage{float}
\usepackage{enumitem}
\usepackage{threeparttable}
\usepackage{multirow}
\usepackage{comment}
\usepackage{tcolorbox}
\usepackage{courier}        
\usepackage{xcolor}
\usepackage{tabularx}
\usepackage{amsthm}
\usepackage{url}
\usepackage{tabularx}
\usepackage{array}
\usepackage{makecell}
\usepackage{ragged2e}
\usepackage{booktabs}

\newcolumntype{Y}{>{\RaggedRight\arraybackslash}X}
\newcolumntype{L}[1]{>{\RaggedRight\arraybackslash}p{#1}}
\newcolumntype{C}[1]{>{\Centering\arraybackslash}p{#1}}

\usepackage{mathrsfs}
\usepackage{amsmath}
\usepackage{amssymb}
\usepackage{mathtools}
\usepackage{amsthm}
\usepackage{amsfonts}
\usepackage{nicefrac}

\usepackage{algorithm}
\usepackage{algorithmic}

\newcommand{\R}{\mathbb{R}}          
\newcommand{\E}{\mathbb{E}}          

\newcommand{\cV}{\mathcal{V}}        

\DeclareMathOperator*{\argmin}{arg\,min}

\newcommand{\Var}{\mathrm{Var}}

\usepackage{wrapfig}
\usepackage{capt-of} 

\theoremstyle{plain}
\newtheorem{theorem}{Theorem}
\newtheorem{corollary}{Corollary}
\newtheorem{proposition}{Proposition}
\newtheorem{lemma}{Lemma}
\newtheorem{assumption}{Assumption}
\newtheorem{definition}{Definition}
\newtheorem{remark}{Remark}

\title{\Large Learning When to Trust LLM Priors: \\ A Validated Framework for Semantic Prior Integration}

\author{%
  Erica Zhang$^{\ast}\,^{1}$ \ Naomi Sagan$^{\ast}\,^{1}$ \ Danny Tse$\,\!^{2}$ \ Fangzhao Zhang$\,\!^{1}$ \
  Mert Pilanci$^\dagger\,^{1}$ \ Jose Blanchet$^\dagger\,^{1}$
}

\newcommand{\affiliations}{
\vskip-4em
\small
\begin{center}
$^{1}$ Stanford University School of Engineering
$^{2}$ Independent Affiliation. \\
\end{center}
\vskip-0.5em
}

\begin{document}

\maketitle

\begingroup
\renewcommand\thefootnote{}
\footnotetext{
$^{\ast}$ Equal contribution
$^{\dagger}$ Equal advising.
Correspondence to: \texttt{\{yz4232, nsagan\}@stanford.edu>}. 
}
\endgroup
\vspace{2em}
\begin{abstract}
Large language models (LLMs) encode rich semantic knowledge that can be useful for supervised learning, but their outputs are unreliable as statistical priors: they may be noisy, misspecified, or hallucinated. Existing LLM-informed learning methods either trust such signals directly, leaving predictions vulnerable to unreliable LLM guidance, or restrict semantic integration to a single model class. We introduce \emph{Statsformer}, a validated framework for learning when to trust LLM-derived semantic priors in supervised statistical learning. Statsformer maps LLM-derived feature scores into a family of learner-specific prior-injection mechanisms across a heterogeneous library of linear and nonlinear predictors. It then uses out-of-fold validation to adaptively calibrate the influence of each prior-informed learner, allowing useful semantic information to improve prediction while attenuating weak, misspecified, or adversarial priors. This yields a guardrailed statistical learning system with an oracle-style guarantee: up to statistical error, the final predictor performs no worse than the best convex combination of its in-library candidates, including prior-free learners. Across diverse prediction tasks, informative LLM priors improve performance, while unreliable priors are automatically downweighted. These results position \emph{Statsformer} as a reliability-oriented approach to LLM-informed statistical learning: rather than trusting LLM knowledge directly, it validates semantic priors against data before allowing them to influence the final predictor.


\end{abstract}

\section{Introduction}
Modern supervised learning often operates in data-scarce, high-dimensional domains with rich 
feature-level semantics, such as genomics and quantitative finance 
\citep{liu2016ski, fang2020priorfinance}. 
Purely data-driven estimators can overfit and select features unstably 
\citep{cantor2024knowledge}, while stabilizing domain priors, such as biological pathways or 
market-structure assumptions, traditionally require laborious human curation and scale poorly 
across modern prediction tasks.

Large language models (LLMs) offer a new way to obtain such domain 
knowledge at scale. 
Trained on broad scientific and natural-language corpora, they can produce feature-relevance 
scores, textual rationales, and other semantic signals, and a rapidly growing body of work 
explores their use as automated knowledge sources in biomedicine, drug discovery, and clinical 
prediction 
\citep{zheng2025llmdrugdiscovery,liu2024llmsmolecularbiology,he2024surveybiomedicalllms,zhang2025llm-lasso}. 
These outputs constitute a novel form of \emph{domain prior}: machine-generated, grounded in broad 
empirical and linguistic regularities, and available at a scale where human expert priors are 
costly to obtain. 
Yet LLM-derived priors are not ordinary expert priors. 
They may be noisy or hallucinated 
\citep{yao2024llmlieshallucinationsbugs,Huang_2024,tian2023justask}, 
and these failure modes are often opaque to the user. This opacity is especially concerning in 
the high-stakes scientific domains where strong priors are also most valuable. 
The central challenge is therefore not only how to \emph{inject} semantic priors into a model, 
but how to learn \emph{when those priors deserve influence}: how can foundation-model knowledge 
be used without assuming that the foundation model is reliable?

Recent work addresses this question through increasingly LLM-centric learning systems. 
Early approaches such as Language-Interfaced Fine-Tuning (LIFT) \citep{dinh2022lift} use 
fine-tuned LLMs directly as tabular predictors. 
Subsequent ML-agent and AutoML-agent systems extend this idea to full learning pipelines, where 
LLMs propose architectures, engineer features, tune hyperparameters, and refine decisions through 
feedback 
\citep{liu2025mlagentreinforcingllmagents,hong2024datainterpreterllmagent,han2024largelanguagemodelsautomatically,tang2024mlbenchevaluatinglargelanguage,zheng2023gpt4performneuralarchitecture,trirat2025automlagentmultiagentllmframework}. 
Related AutoML work similarly inserts natural-language reasoning into pipeline search and 
optimization, including feature engineering and hyperparameter selection 
\citep{tornede2024automlagelargelanguage,xu2024largelanguagemodelssynergize,hollmann2023largelanguagemodelsautomated,li2024exploringlargelanguagemodels,liu2024largelanguagemodelsenhance,zhang2024mlcopilotunleashingpowerlarge}. 
These systems are flexible and powerful, but they place the LLM in the role of an active 
decision-maker. 
As a result, model quality depends on the reliability of the LLM's choices, with safeguards 
provided mainly through prompting, fine-tuning, or costly feedback loops.

Our work takes a different stance. 
Rather than building full learning pipelines around the LLM, we focus on the statistical learning 
stage itself and treat LLM-derived semantics as candidate priors whose influence must be validated 
by data. 
This reframes LLM-informed learning as a problem of \emph{statistical calibration of semantic 
trust}: the influence of an LLM-derived prior should be determined by empirical evidence rather 
than procedural heuristics or assumed fidelity. This perspective builds on a long statistical tradition. 
Bayesian priors, weighted regularization, and hierarchical penalties all incorporate prior 
knowledge while recognizing that priors of uncertain quality should not be trusted blindly 
\citep{gelman2013bayesian,tibshirani1996regression,zou2006adaptive}. 
Most directly, our approach follows ensemble aggregation theory and the super-learner framework, 
which combine heterogeneous learners through cross-validated weights and provide oracle-style 
performance guarantees relative to the best candidate, up to statistical error 
\citep{dalalyan2012sharp,rigollet2012sparse,vanderlaan2007superlearner}. 
For LLM-informed learning, this principle becomes essential: semantic priors should be used when 
informative, discounted when unreliable, and integrated in a way that remains compatible with 
diverse statistical learners.

Existing attempts to integrate language-model priors into statistical learning fall short of this 
goal in either reliability or scope.
Methods such as LM-Priors \citep{choi2022lmpriorspretrainedlanguagemodels} and LLM-Select 
\citep{jeong2024llmselectfeatureselectionlarge} inject LLM relevance scores directly into 
downstream models without validating their fidelity, leaving them prone to hallucinated or 
adversarial signals. 
More principled formulations such as LLM-Lasso \citep{zhang2025llm-lasso} calibrate LLM-derived 
penalty weights against data, but are confined to a single linear learner and focus primarily on 
feature selection. 
As a result, current approaches offer either unvalidated flexibility or validated but narrow 
scope and none provides a general, data-corroborated treatment of LLM semantics across diverse 
supervised learners with formal performance guarantees.

In this paper, we introduce \emph{Statsformer}, a validated framework for incorporating 
LLM-derived semantic priors into supervised statistical learning. 
Statsformer maps LLM-derived feature scores into learner-specific prior-injection mechanisms 
across heterogeneous linear and nonlinear predictors, then combines prior-informed and prior-free 
candidates through out-of-fold stacking. 
Thus, semantic priors propose inductive biases, while validation data determine whether to trust, 
attenuate, or ignore them. Our key contributions are as follows:
\begin{itemize}[leftmargin=*, nosep]
  \item \textit{\textbf{Validated semantic-prior integration across model classes.}}
  \emph{Statsformer} is a model-agnostic aggregation framework that integrates LLM-derived 
  feature priors into supervised learning across diverse base learners and prior-injection 
  mechanisms, unified through a monotone prior-injection family. 
  It requires only a single LLM query, with no fine-tuning, iterative prompting, or inference-time 
  LLM calls, making it efficient and scalable across language-model choices.

  \vspace{0.5em}
  \item \emph{\textbf{Provable safety under unreliable priors.}}
  Using classical convex aggregation theory, we show that Statsformer competes with any convex 
  combination of candidate learners, including prior-free learners, up to statistical error. 
  To our knowledge, this is the first model-agnostic validated framework for LLM-derived feature 
  priors with such a downstream performance floor, ensuring graceful degradation when the LLM 
  signal is weak, misspecified, or adversarial.

  \vspace{0.5em}
  \item \emph{\textbf{Stress-tested robustness and quantified upside.}}
We evaluate both the safety floor and potential upside of the framework. 
Adversarial-prior experiments show that misleading LLM signals are downweighted toward 
prior-free performance, while oracle-prior simulations quantify the headroom available from 
stronger semantic priors beyond the no-prior baseline.
\end{itemize}

\section{Preliminaries: Semantic Priors and Ensemble Learning}\label{sec:prelim}
\vspace{-0.5em}
\paragraph{Supervised learning with semantic priors.}
We consider supervised learning from labeled samples 
$\mathcal{D}=\{(x_i,y_i)\}_{i=1}^n$, where $x_i\in\mathbb{R}^p$ and 
$y_i\in\mathcal{Y}$.\footnote{In classification, $\mathcal{Y} \subset \mathbb{N}$ and 
$|\mathcal{Y}| < \infty$. In regression, $\mathcal{Y} \subset \mathbb{R}$.} 
The goal is to learn a predictor $f_\theta:\mathbb{R}^p\to\mathcal{Y}$ that minimizes 
the population risk 
$R(f_\theta)\coloneqq \mathbb{E}[\ell(f_\theta(X),Y)]$, approximated by the 
empirical risk 
\begin{align}\label{eq:erm}
\hat R_n(f_\theta) \coloneqq \frac{1}{n}\sum_{i=1}^n \ell(f_\theta(x_i),y_i),
\end{align}
optionally with a regularization term $\lambda\,\Omega(\theta)$ 
\citep{tibshirani1996lasso,zou2005elastic}. 
We denote LLM-derived feature relevance by a \emph{semantic prior vector} 
$V=(v_1,\ldots,v_p)$, where $v_j$ represents the semantic relevance assigned to feature $j$.
\vspace{-0.8em}
\paragraph{Ensemble learning and stacking.}
An ensemble combines a finite collection of predictors $\{\hat f_j\}_{j=1}^J$ through an 
aggregation rule $F$:
$
\hat f(x) \coloneqq F(\hat f_1(x), \ldots, \hat f_J(x)).
$

\emph{Homogeneous ensembles} use a common learning algorithm and create diversity through 
randomization, as in bagging 
\citep{breiman1996bagging,breiman2001randomforest} and boosting 
\citep{freund1996experiments,friedman2001gbm}. 
\emph{Heterogeneous ensembles}, such as stacking 
\citep{wolpert1992stacked,leblanc1996combining,vanderlaan2007super}, combine learners with 
different algorithms and inductive biases, with the aggregation rule fit using out-of-fold 
predictions. 
Statsformer builds on this heterogeneous regime: its candidate learners span both prior-informed 
and prior-free variants across linear and nonlinear model classes, and their aggregation weights 
are learned from out-of-fold validation performance.

\section{The Statsformer Framework}\label{sec:framework}
\vspace{-0.5em}
\begin{wrapfigure}{r}{0.45\textwidth}
\vspace{-\intextsep}  
\begin{minipage}{\linewidth}
\begin{algorithm}[H]
\caption{Statsformer}
\label{alg:statsformer}
\begin{algorithmic}[1]
\footnotesize
\REQUIRE Data $\mathcal{D} :=\{x_i,y_i\}_{i=1}^n$; feature prior $V$; base learners $\mathcal{B}=\{b_m\}_{m=1}^M$; hyperparameter grids per base learner $\{\Theta_m\}$; number of folds $K$.
\STATE \textbf{OOF stage:}
\FOR{$k=1$ \textbf{to} $K$}
    \FOR{$m=1$ \textbf{to} $M$}
        \FOR{$\theta\in\Theta_m$}
            \STATE {Fit $b_m$ with adapter \\ $A_m(b_m;\mathcal{D}_{(-k)};\tau_\alpha(V))$}
            \STATE Store out-of-fold predictions.
        \ENDFOR
    \ENDFOR
\ENDFOR
\STATE \textbf{Meta stage:} Fit weights $\hat{\pi} \in \Delta$ on OOF predictions.
\STATE \textbf{Refit:} Train base learners on full data and output
\[
\hat f_{\mathrm{SF}}(x)=\sum_{l}\hat\pi_l \hat f_l(x).
\]
\end{algorithmic}
\end{algorithm}
\end{minipage}
\vspace{-1em}
\end{wrapfigure}
\emph{Statsformer} integrates LLM-derived feature priors into supervised learning via a principle we term \emph{validated prior integration}.
In this framework, semantic priors supplied by a foundation model act as an \emph{inductive bias} for supervised learning models, while empirical risk validation determines how strongly we should rely on that bias.
Specifically, we consider an ensemble of statistical learning methods (e.g., Lasso, XGBoost), referred to as \emph{base learners}. A foundation model provides an external feature-level prior
$
V = (v_1,\dots,v_p) \in [0,\infty)^p .
$
For each base learner, the prior $V$ is injected into the learning process (e.g., via feature weights or weighted regularization). The strength of the prior is controlled by a monotone transformation of $V$ parameterized by $\alpha \ge 0$: when $\alpha = 0$, the prior is uniform, and larger values of $\alpha$ impose stronger prior influence. We consider $\alpha \in \mathscr{A}$, a finite set that includes $\alpha=0$.

We then construct a stacking ensemble over all base learners and prior strengths $\alpha \in \mathscr{A}$, fit via cross-validation. The final predictor is a convex combination of the ensemble components, with data-driven weights that calibrate the influence of semantic priors relative to prior-free learners.

This explicit separation between \emph{prior injection} and \emph{prior validation} is central to Statsformer. Semantic priors shape the geometry of individual learners, while their downstream impact is governed by out-of-fold predictive performance. Consequently, the framework exploits informative priors when they align with the data and automatically attenuates or ignores unreliable ones, enabling a principled integration of foundation-model knowledge into statistical learning.
\subsection{Prior Injection via Monotone Transformations.}\label{subsec:monotone-transformations}
Typically, methods of prior injection into a base learner expect \emph{feature penalties} $\{w_j\}_{j=1}^p$, or \emph{feature weights} $\{s_j\}$.
The further $\{w_j\}$ or $\{s_j\}$ is from a uniform vector, the stronger the inductive bias placed on the learner.

\textbf{Monotone map.}
To modulate the influence of an external prior $V$, Statsformer defines a \emph{monotone, temperature-controlled}\footnote{We provide theoretical justification for monotonicity in the prior injection family in Appendix~\ref{appdx:monotone_map_families}.} family of maps
\[
\tau_\alpha : [0,\infty) \to (0,\infty), \qquad \alpha \ge 0,
\]
satisfying the null condition $\tau_0(v) = 1$ for all $v$. The map $\tau_\alpha$ serves as a generic transformation of prior scores; different base learners may adopt different instantiations of this transformation.

For each feature $j$, the prior score $v_j$ is transformed via $\tau_\alpha(v_j)$ before being injected into the learning procedure. Common instantiations include
\[
w_j(\alpha) = (v_j + \epsilon)^{-\alpha}, \qquad
s_j(\alpha) = v_j^{\alpha},
\]
where $\epsilon > 0$ is a small constant for numerical stability. Here, $w_j(\alpha)$ and $s_j(\alpha)$ represent weight-type and scale-type realizations of the same underlying map $\tau_\alpha$.

\textbf{Injection mechanisms.}
Each base learner $b_m\in\mathcal{B}$ is paired with an \emph{adapter}\footnote{Adapters specify 
how priors are incorporated into learners; guidance on their selection is provided in 
Appendix~\ref{appdx:base_learner}.}, which determines how the transformed prior values 
$\{\tau_\alpha(v_j)\}_{j=1}^p$ enter training. 
Although the concrete injection mechanism is model-specific, Statsformer imposes a common 
principle: the prior enters only through a monotone rescaling at a single, well-defined interface 
of the learning objective. 
This abstraction is model-agnostic and supports diverse learner architectures. 
In practice, it yields three broadly applicable instantiations (see 
Appendix~\ref{appdx:prior-injection} for details):

\begin{enumerate}
    \item \textit{Penalty-based injection:} for generalized linear models, the prior rescales feature-specific regularization strengths, modifying the geometry of the penalty, i.e., the regularizer is $\lambda \sum_{i=1}^p w_i(\alpha) \phi(\theta_i)$, where $\phi$ is a coordinate-wise regularization function.
    \item \textit{Feature-reweighting injection:} for learners whose feature importance is sensitive to scaling, the prior rescales input features prior to training (e.g., replacing input feature $x_{i,j}$ with $x_{i,j} s_j(\alpha)$), thereby amplifying or attenuating their influence throughout optimization.
    For tree-based methods, this same principle manifests through feature sampling proportions.
    \item \textit{Instance-weight injection:} for learners that accept observation weights, the prior induces sample weights that emphasize training examples in which prior-identified features are active.
    Specifically, each term in the empirical risk of \eqref{eq:erm} is scaled by sample weight $\rho_i(\beta)$, where $\beta \in [0, 1]$ modulates the strength of the weighting ($\beta=0$ means uniform weighting). This is a stable first-order approximation to an exponential-tilt solution (Appendix \ref{subsec:kl-instance}).
\end{enumerate} 

We denote the injection mechanism by $A_m(b_m,\mathcal{D};\tau_\alpha(V))$, and when instance-weight injection applies we write $A_m(b_m,\mathcal{D},\beta;\tau_\alpha(V))$. 
In all cases, the prior-free learner is recovered at $(\alpha,\beta)=(0,0)$.
Thus, each learner family contains a prior-free baseline.

\subsection{Prior Validation With Out-Of-Fold (OOF) Stacking}
\label{subsec:oof} 

Statsformer calibrates the influence of semantic priors through an out-of-fold (OOF) stacking procedure that jointly validates prior strength, model identity, and ensemble weights.
For each base learner $b_m \in \mathcal{B}$, let $\Theta_m$ denote a finite hyperparameter grid\footnote{To reduce computational overhead, some hyperparameters may be, in practice, selected separately from the OOF stacking.} \emph{augmented} with prior-related parameters $\alpha \in \mathscr{A}$ and, when applicable, instance-weight parameters $\beta \in
\mathscr{B}$.
The null configuration the prior-free learner, $(\alpha,\beta)=(0,0)$, is always included in $\Theta_m$.

We perform this stacking via $K$-fold cross-validation.
For each fold $k \in \{1,\dots,K\}$ and each learner
configuration $(m,\theta) \in \{1,\dots,M\}\times\Theta_m$, we fit $b_m$ on the
training folds $\{1,\dots,K\}\setminus\{k\}$ using the adapter $A_m(\cdot;\tau_\alpha(V))$ (with, when applicable, instance weights $\rho_i(\beta)$).
Let superscript $(-k)$ denote training with fold $k$ 
removed, i.e., the predictor fit on $\mathcal{D}\setminus \mathcal{D}_k$. This yields a fitted model
$f_{m,\theta}^{(-k)}$.
We then form out-of-fold (OOF) predictions
$
z_{i,(m,\theta)} = f_{m,\theta}^{(-k)}(x_i)$, where $i \in \text{fold } k.
$

Collecting all OOF predictions produces a design matrix
$
Z \in \mathbb{R}^{n \times L}$ with 
$L = \sum_{m=1}^M |\Theta_m|$,
where each column corresponds to a distinct learner configuration (i.e., base learner and value of $(\alpha, \beta)$).
For notational convenience, we index configurations by a single index
$l \in \{1,\dots,L\}$, with the correspondence $l \;\longleftrightarrow\; (m(l),\theta(l))$
and write $z_l \in \mathbb{R}^n$ for the OOF prediction vector associated with
configuration $l$.

We aggregate these predictions by fitting a simplex-constrained
meta-learner.
For regression tasks, we use non-negative least squares; for classification,
we operate on model logits and apply $\ell_2$-regularized logistic regression.
Formally, the aggregation weights are obtained by solving
\begin{align}
\hat{\pi}
=
\argmin_{\pi \in \Delta^{L-1}}
\hat{R}_n\!\left(
\sum_{l=1}^L \pi_l z_l,\,
Y
\right),
\label{eq:meta-learner}
\end{align}
where $\hat{R}_n$ is defined as in Equation \ref{eq:erm}.
After estimating $\hat{\pi}$, each learner configuration indexed by $l \in \{1,...,L\}$ is refitted on the
full dataset, resulting in the fitted model $\hat f_l.$
The final Statsformer predictor is\footnote{For multiclass classification, we train the out-of-fold stacking meta-learner using a one-versus-rest decomposition, fitting one binary meta-model per class.}
\begin{align}
\hat f_{\mathrm{SF}}(x)
=
\sum_{l=1}^L \hat{\pi}_l \,\hat f_l(x).
\label{eq:final-predictor}
\end{align}


OOF aggregation provides a built-in guardrail: since $(\alpha,\beta)=(0,0)$ is included in the dictionary, Statsformer assigns weight to prior-modulated learners only when they improve out-of-fold risk relative to the baseline.
Algorithm~\ref{alg:statsformer} summarizes the procedure, with theoretical guarantees in Section~\ref{sec:theory} and computational analysis in Appendix~\ref{app:comp_analysis}.

Appendix~\ref{appdx:empirical_compute} reports the empirical computational cost of Statsformer.
While Statsformer incurs overhead relative to the base methods, this overhead is modest in practice, and, especially in high-dimensional settings, lower than that of common AutoML systems.

\section{The Statistical Floor: Oracle Guarantee}\label{sec:theory}
\vspace{-0.5em}
Statsformer performs convex aggregation over a finite, data-dependent dictionary of predictors produced by (i) prior injection and (ii) out-of-fold stacking.
This section proves oracle inequalities for the deployed (refit) aggregate and explains the resulting robustness to misspecified priors.

Let $\mathcal{F}=\{f_l\}_{l=1}^L$ denote a finite collection of candidate models obtained from the base library $\mathcal{B}$ by enumerating hyperparameters and prior-strength parameters. We partition the data into $K$ folds, indexed by $\{I_k\}_{k=1}^K$ with $|I_k|=n_k$. For each configuration $l$ and fold $k$, let $f_l^{(-k)}$ be the predictor trained on all data excluding fold $k$. As before,
$\hat f_l$ denotes the corresponding predictor refit on the full sample.
For any aggregation weight vector $\pi\in\Delta^{L-1}$, define the cross-validated empirical risk
\begin{equation}
\hat R_{\mathrm{CV}}(\pi)
\coloneqq
\frac{1}{n}\sum_{i=1}^n
\ell\!\left(\sum_{l=1}^L \pi_l f_l^{(-k(i))}(x_i),\, y_i\right),
\label{eq:cv-risk}
\end{equation}
where $k(i)$ is the fold containing index $i$. Statsformer selects aggregation weights $\hat\pi \in \arg\min_{\pi\in\Delta^{L-1}} \hat R_{\mathrm{CV}}(\pi).$
The predictor uses refit base models $\hat f_l$ with population risk\footnote{Throughout this section, expectations are conditional on $\mathcal{D}$ and taken over an independent test point $(X,Y) \sim P$.}

\begin{equation}
R_{\mathrm{refit}}(\pi)
\coloneqq
\mathbb{E}\!\left[\ell\!\left(\sum_{l=1}^L \pi_l \hat f_l(X),\, Y\right)\right].
\label{eq:refit-risk}
\end{equation}
To connect the empirical CV objective to population performance, we introduce an analysis surrogate
based on OOF predictors, i.e., the cross-fitted population risk:
\begin{equation}
R_{\mathrm{CF}}(\pi)
\coloneqq
\sum_{k=1}^K \frac{n_k}{n}\,
\mathbb{E}\!\left[\ell\!\left(\sum_{l=1}^L \pi_l f_l^{(-k)}(X),\, Y\right)\right].
\label{eq:cf-risk}
\end{equation}
Equivalently, define the randomized cross-fitted aggregate $\tilde f_{\mathrm{SF}}$ that, for a given
input $X$, draws $k$ with probability $n_k/n$ and outputs $\sum_{l=1}^L \hat\pi_l f_l^{(-k)}(X)$.
By construction,
\begin{equation}
R(\tilde f_{\mathrm{SF}})=R_{\mathrm{CF}}(\hat\pi)
\qquad\text{(see Lemma~\ref{lemma:risk_id}).}
\label{eq:risk-id}
\end{equation} 
This identity lets us first bound $R_{\mathrm{CF}}(\hat\pi)$ and then transfer the guarantee to the
deployed refit predictor via a small refit gap. We are now ready to state our main results.

\begin{theorem}[Oracle Guarantees for Validated Prior Integration]
\label{thm:convex-aggregation}
Assume $\ell(\cdot, y)$ is convex, $L_\ell$-Lipschitz, and $\ell_{\mathrm{max}} \coloneqq \sup_{y \in \mathcal Y} |\ell(0,y)|$ is bounded.
Furthermore, assume that all cross-fitted predictors satisfy $\|f_l^{(-k)}(x)\|\le B$ for all $k, l, x$. Then, for any $\delta\in(0,1)$, with probability at least $1-\delta$,
\begin{equation}
R_{\mathrm{CF}}(\hat\pi) \le \inf_{\pi\in\Delta^{L-1}} R_{\mathrm{CF}}(\pi) + C L_\ell B\sqrt{\frac{K\log(\frac{LK}{\delta})}{n}},
\label{eq:oracle-convex-cf}
\end{equation}
where $C>0$ is a universal constant.\footnote{In practice $K$ is fixed and small, so the $K$-dependent factor is $O(1)$ and the rate is $O\!\big(L_\ell B\sqrt{\log(L/\delta)/n}\big)$.}
\end{theorem}

\begin{proof}
    See Appendix \ref{appdx-subsec:oracle}.
\end{proof}

We now extend this result to the refit Statsformer predictor used in deployment.

\begin{corollary}[Oracle Guarantee for the Refit Statsformer Predictor]
\label{cor:refit}
Assume the conditions of Theorem~\ref{thm:convex-aggregation} hold. Define the refit gap $\Delta_{\mathrm{refit}} := \sup_{\pi\in\Delta^{L-1}} |R_{\mathrm{refit}}(\pi) - R_{\mathrm{CF}}(\pi)|$ and assume $\Delta_{\mathrm{refit}}\le \varepsilon_n$. Then, with probability $\geq1-\delta$,
\begin{equation}
R_{\mathrm{refit}}(\hat\pi)\le\inf_{\pi\in\Delta^{L-1}} R_{\mathrm{refit}}(\pi) + C L_\ell B \sqrt{\frac{K\log(\frac{LK}{\delta})}{n}} + 2\varepsilon_n.
\label{eq:oracle-refit}
\end{equation}
\end{corollary}

\begin{proof}
    See Appendix \ref{appdx:refit-proof}. Sufficient conditions ensuring $\varepsilon_n \to 0$ are discussed in Appendix~\ref{appdx:refit-suff}.
\end{proof}

The oracle bound formalizes Statsformer's \emph{statistical floor}: 
Theorem~\ref{thm:convex-aggregation} shows that the learned aggregate competes with the best 
convex combination of dictionary predictors, up to statistical error. 
Thus, informative priors can improve performance, while faulty priors are downweighted and cannot 
harm performance beyond this estimation error.\footnote{When boundedness or Lipschitz conditions 
are relaxed, e.g., squared-error regression with sub-Gaussian noise, analogous model-selection 
oracle guarantees still hold; see Appendix~\ref{subsec:extension}.}

\section{Experiments}\label{sec:experiment}
\vspace{-0.5em}
\begin{wraptable}{r}{0.6\textwidth}
\vspace{-\intextsep}
\centering
\scriptsize
\begin{tabular}{lccc}
\toprule
Dataset & Features & Instances & Reference \\
\midrule
Breast Cancer & 1,000 & 1,545 & \cite{stiglic2010gemler} \\
Bank Marketing & 16 & 1,000 & \cite{bank_marketing_222} \\
ETP & 1,000 & 189 & \cite{liu2017genomic} \\
Credit & 20 & 1,000 & \cite{uci_german_credit} \\
Internet Ads & 1,555 & 1,000 & \cite{internet_advertisements_51} \\
Lung Cancer & 1,000 & 1,017 & \cite{weinstein2013cancer} \\
Nomao & 120 & 1,000 & \cite{uci_nomao} \\
Superconductivity & 81 & 1,000 & \cite{uci_superconductivity} \\
\bottomrule
\end{tabular}
\caption{\scriptsize Summary of datasets used in our experiments. Features indicate post-processing counts; instances reflect subsampling when applied. See Appendix~\ref{appdx:datasets} for dataset processing specifics. All datasets are binary classification, except Superconductivity, which is regression.}
\label{tab:datasets}
\end{wraptable}
Section~\ref{result_main} compares Statsformer with strong tabular baselines, including 
AutoML systems; Section~\ref{sec:exp-floor} stress-tests its safety floor under weak or 
adversarial priors; and Section~\ref{sec:exp-ceiling} quantifies the headroom available from 
informative priors beyond the no-prior baseline.
We evaluate Statsformer on tabular datasets with semantically meaningful feature names, varying 
training-set size to test how semantic priors help across data regimes. 
\vspace{-2em}
\subsection{Implementation Details}\label{sec:implementation-details}
\vspace{-0.5em}
We summarize the Statsformer implementation, datasets, and baselines. Full details are in Appendix~\ref{appdx:experimental-details}.

\textbf{Datasets and Data Splitting.}
Our experiments span a diverse set of benchmark datasets across domains including gene expression, marketing, finance, and web analytics, with emphasis on high-dimensional, low-sample regimes.
Dataset details are provided 
\begin{figure*}[t]
    \centering
    \includegraphics[width=1\linewidth]{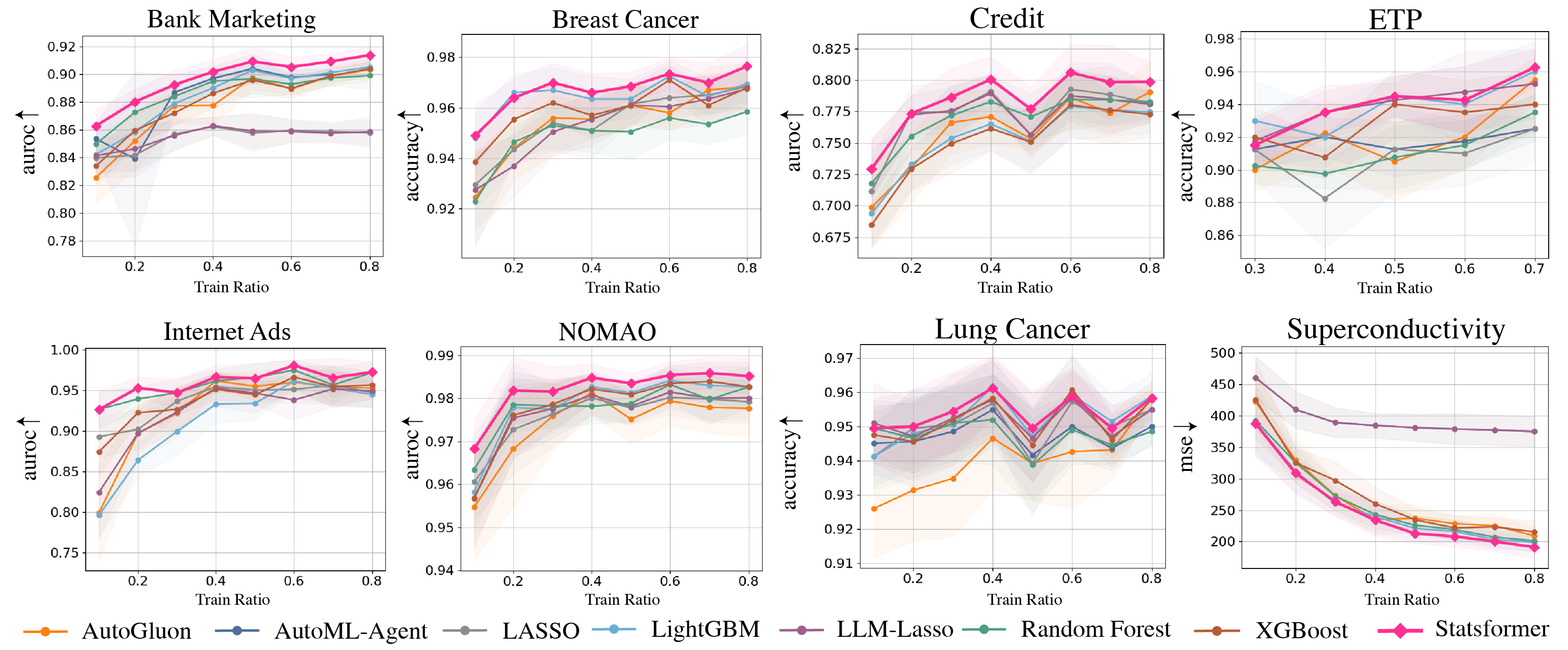}
    \caption{\scriptsize Statsformer performance on a variety of datasets, compared to a variety of baseline methods.
    Note that, due to computational constraints, we only included the AutoML-Agent baseline in Bank Marketing, ETP, and Lung Cancer (see Table~\ref{tab:computation} in the Appendix for a more detailed computational comparison).
    For all datasets, we plot either accuracy or AUROC, where higher is better, except Superconductivity, where we plot mean squared error (lower is better).
    For each training ratio, we plot the mean of the selected metrics 10 different train-test splits (selected via stratified splitting), as well as 95\% confidence intervals.
    Due to the low-sample and imbalanced nature of ETP, we limit the training sizes to be in between $0.3$ and $0.7$ to allow sufficient positive samples in each training and test split.}
    \label{fig:baseline-comparison}
    \vspace{-1.5em}
\end{figure*}
in Appendix~\ref{appdx:datasets}, with a summary in Table~\ref{tab:datasets}. We evaluate performance across training set sizes by subsampling the training data and averaging results over $10$ random splits per training ratio.

\textbf{Statsformer Implementation.}
We use four strong base learners that admit feature-level prior injection via the adapters in Section~\ref{sec:framework}: Lasso \citep{tibshirani1996lasso} (penalty factors), XGBoost \citep{chen2016xgboost} (feature sampling weights), random forests \citep{breiman2001randomforest} (feature sampling and instance weights), and kernel SVMs \citep{Zhang2011WeightedSVM} (feature scaling).
Details of prior injection and per-learner gains are given in Appendix~\ref{appdx:base_learner}.


\begin{figure*}[htbp]
    \centering
    \begin{minipage}[t]{\textwidth}
        \centering
        \includegraphics[width=0.85\linewidth]{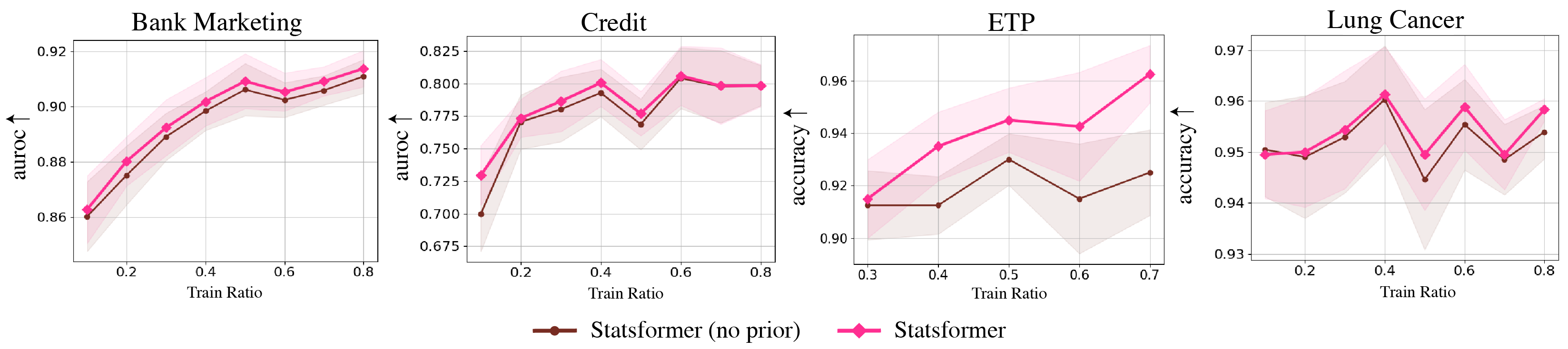}
        \caption{\scriptsize Direct accuracy and AUROC comparison of \textbf{Statsformer} to \textbf{Statsformer (no prior)} for selected datasets.
        Gains are noticeable across all four examples, and significant for ETP.
        See Figure~\ref{fig:appdx_us_vs_stacking} in the Appendix for datasets not shown here.}
        \label{fig:us_vs_stacking}
    \end{minipage}
    \vspace{1em}
    
    \begin{minipage}[t]{\textwidth}
    \centering
    \tiny
    \setlength{\tabcolsep}{5pt}
    \begin{tabular}{@{}llcccccccc@{}}
    \toprule
     && \textbf{Bank} & \textbf{Breast} & \textbf{Credit} & \textbf{ETP} & \textbf{Internet} & \textbf{Lung} & \textbf{NOMAO} & \textbf{Supercond.} \\
    \midrule
    \multirow{2}{*}{\textbf{Imp. (\%)}} & \textbf{Error} 
    & 1.40 ± 0.71 
    & 4.55 ± 3.11 
    & 1.52 ± 1.22 
    & \textbf{25.9 ± 8.87} 
    & 4.09 ± 5.10 
    & 4.20 ± 3.27 
    & 0.00 ± 2.34 
    & 1.37 ± 0.81 \\
    & \textbf{AUC} 
    & 3.09 ± 0.90 
    & 2.45 ± 4.18 
    & 3.19 ± 1.69 
    & \textbf{45.4 ± 17.4} 
    & \textbf{8.53 ± 4.86} 
    & \textbf{12.2 ± 4.65} 
    & 3.52 ± 2.64 
    & --- \\
    \midrule
    \multirow{2}{*}{\textbf{Win Rate}} & \textbf{Error} 
    & 0.70 ± 0.10 
    & 0.84 ± 0.08 
    & 0.73 ± 0.10 
    & \textbf{0.92 ± 0.08} 
    & 0.69 ± 0.10 
    & \textbf{0.79 ± 0.09} 
    & 0.64 ± 0.11 
    & 0.73 ± 0.10 \\
    & \textbf{AUC} 
    & \textbf{0.78 ± 0.09} 
    & 0.60 ± 0.11 
    & 0.71 ± 0.10 
    & \textbf{0.88 ± 0.09} 
    & 0.68 ± 0.10 
    & \textbf{0.76 ± 0.09} 
    & 0.69 ± 0.10 
    & --- \\
    \bottomrule
    \end{tabular}
    \vspace{0.5em}
    \captionof{table}{\scriptsize
    \textit{Top}: Mean performance improvement of \textbf{Statsformer} over \textbf{Statsformer (no priors)}, reported with $\pm$ 95\% confidence intervals. Improvements are computed as the mean metric difference divided by the mean baseline error (i.e., Error or $1-\text{AUROC}$), and are expressed as percentages.
    \textit{Bottom}: Win rate of \textbf{Statsformer} over \textbf{Statsformer (no priors)} with $\pm$ 95\% confidence intervals.
    Win rate is defined as the proportion of splits in the dataset where Statsformer performs at least as well as the no-prior baseline.}
    \label{tab:statsformer_full_summary}
\end{minipage}
\vspace{-1em}
\end{figure*}

For each base learner, we sweep $\alpha\in\{0,1,2\}$ and $\beta\in\{0.75,1\}$.
To limit computation, no additional hyperparameter tuning is performed during OOF stacking (e.g., Lasso selects its regularization internally).
The meta-learner is non-negative $\ell_2$-regularized logistic regression for classification and elastic net for regression.
For the main results, we use the OpenAI \texttt{o3} reasoning model \citep{openai2025o3o4mini} to generate feature scores.
Model ablations use additional open and proprietary LLMs (Appendix~\ref{model_ablation_append}).
Inference is run via \texttt{OpenAI} and \texttt{OpenRouter} APIs.

\textbf{Baseline Details.}
We compare Statsformer against strong tabular learning baselines, including both classical methods and AutoML systems.\footnote{
Statsformer is not an end-to-end AutoML system and does not perform architecture or pipeline search; AutoML methods serve as widely used reference points.
}
Among AutoML approaches, we include AutoGluon~\cite{erickson2020autogluontabularrobustaccurateautoml}, a plug-and-play ensemble of neural networks and tree-based models, and AutoML-Agent~\cite{trirat2025automlagentmultiagentllmframework}, a multi-agent LLM framework for full-pipeline AutoML. We also benchmark standard statistical methods, including XGBoost~\cite{chen2016xgboost}, Random Forests~\cite{breiman2001randomforest}, LightGBM~\cite{ke2017lightgbm}, and Lasso~\cite{tibshirani1996lasso}, as well as \emph{LLM-Lasso}~\cite{zhang2025llm-lasso}, which incorporates LLM-derived feature penalties into Lasso.

\subsection{Main Results}\label{result_main}

\textbf{Comparison with Baselines.}
Figure~\ref{fig:baseline-comparison} compares Statsformer with the baselines and datasets from Section~\ref{sec:implementation-details}.
Across datasets, Statsformer is consistently the top-performing method (or a close second), while baseline performance varies widely.
In particular, Statsformer outperforms both AutoML benchmarks, i.e. AutoGluon and AutoML-Agent, highlighting its practical competitiveness and robustness.
In contrast, AutoML-Agent exhibits occasional instability; for one split of the Bank Marketing dataset, a preprocessing error led to severe overfitting and a sharp AUROC drop at \texttt{train\_ratio=0.2}. For clarity, we report either accuracy or AUROC per dataset; remaining metrics appear in Figure~\ref{fig:appdx_us_vs_baselines}, and follow the same trends.

\begin{figure*}[htb!]
    \centering
    \includegraphics[width=0.9\linewidth]{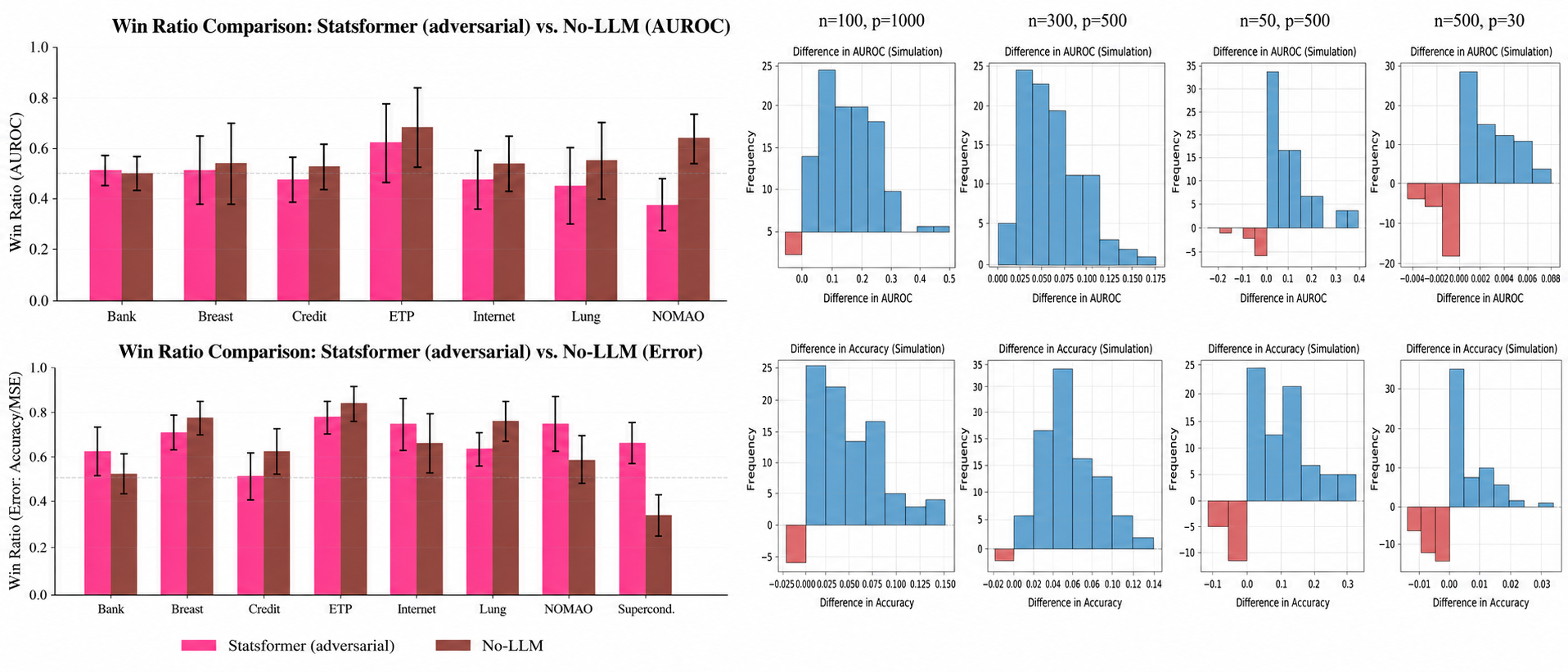}
    \caption{\scriptsize \textbf{Floor and ceiling.}
\textit{Left} (floor): Win ratios of \textbf{adversarial-prior Statsformer (pink)} versus the 
\textbf{no-prior baseline (brown)} across datasets, computed as the percentage of train--test 
splits where one method performs at least as well as the other. 
Win ratios cluster near $0.5$ across both AUROC (top) and Error (bottom), confirming that 
corrupted priors degrade gracefully to no-prior performance.
\textit{Right} (ceiling): Distributions of AUROC (top) and accuracy (bottom) differences 
between Statsformer with a perfectly informative oracle prior and the no-prior baseline, 
across $(n,p)$ regimes. 
Substantial headroom is available in underdetermined regimes ($p \gg n$); headroom shrinks 
toward zero in the well-determined regime ($n=500, p=30$), consistent with classical 
asymptotic efficiency.
    }
    \label{tab:inverted_priors_win_ratio}
\end{figure*}

\begin{figure*}[htb!]
    \centering
    \includegraphics[width=1\linewidth]{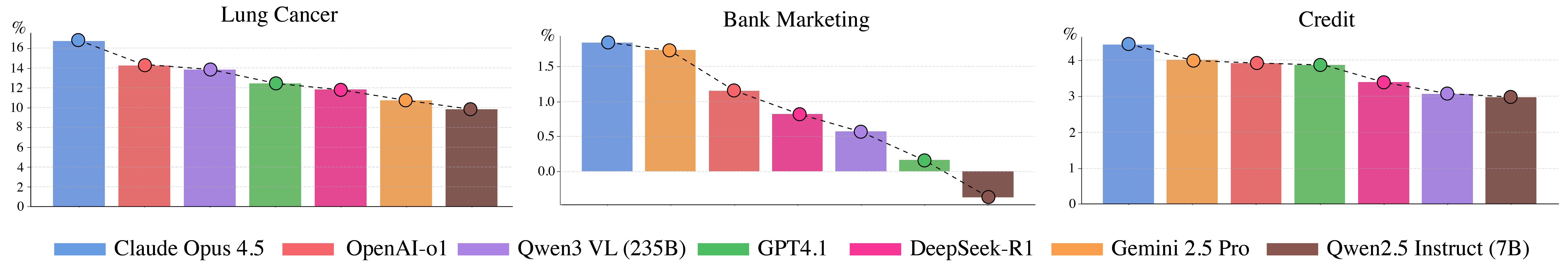}
    \caption{\scriptsize Mean performance improvement of \textbf{Statsformer} over \textbf{Statsformer (no priors)}, with prior scores generated by various LLM choices. We present more about experimental setting and additional results in Appendix \ref{model_ablation_append}. For all datasets, we plot AUROC, where higher is better. Qwen2.5 Instruct (7B) is (arguably) the weakest LLM among all choices, whereas Claude overall performs well.}
    \label{fig:model_ablation_main}
\end{figure*}

\textbf{Ablation of LLM-Generated Priors.}

We ablate LLM-generated priors to form \textsc{Statsformer (no-prior)}, a plain out-of-fold 
stacking baseline using only $\alpha=\beta=0$ configurations. 
This is a strong Super Learner-style baseline that often outperforms individual learners 
\citep{vanderlaan2007super,Polley2010super}. 
Table~\ref{tab:statsformer_full_summary} shows that Statsformer improves over it across most 
dataset-metric pairs: win-ratio lower confidence bounds exceed $50\%$ in nearly all cases, and 
relative improvements are usually positive, with the largest gains in underdetermined settings 
such as ETP and Lung Cancer. 
Representative accuracy and AUROC comparisons appear in Figure~\ref{fig:us_vs_stacking}, with 
additional results in Figure~\ref{fig:appdx_us_vs_stacking}. 
These gains are not due to extra ensembling or prompt tuning. 
We use a fixed lightweight prompting pipeline with no tuning or human feedback, and inverted 
feature priors return performance near the prior-free stacking baseline 
(Section~\ref{sec:simulation}), suggesting that Statsformer exploits genuine, though sometimes 
weak, semantic signal.

\textbf{Model Ablations.}
To evaluate the robustness of our approach to the choice of LLM used for score generation, we repeat our experiments using scores produced by a range of LLMs.
Specifically, we consider both open-source frontier models, including the Qwen series \citep{qwen3technicalreport} and DeepSeek-R1 \citep{deepseek_r1}, as well as proprietary models, including Anthropic’s Claude Opus~4.5 \citep{anthropic2025claudeopus}, OpenAI’s \texttt{o1} \citep{openai_gpt_o1}, and Google’s Gemini~2.5 \citep{google_gemini_2.5}.
Figure~\ref{fig:model_ablation_main} reports the same statistics as Table~\ref{tab:statsformer_full_summary}, with feature scores provided by each LLM.

Across model choices, Statsformer consistently improves over stacking, with occasional 
exceptions for smaller models (e.g., Qwen2.5-7B). 
Larger and more recent LLMs tend to yield stronger gains, suggesting performance is partly 
bounded by prior quality. For example, Claude exceeds \texttt{o3} on Lung Cancer and ETP, 
indicating that Figure~\ref{fig:baseline-comparison} represents current rather than ceiling 
performance. We quantify this dependence directly in Section~\ref{sec:exp-ceiling}. 
Additional experimental details and results are provided in Appendix~\ref{model_ablation_append}.


\vspace{-0.5em}
\subsection{The Floor: Robustness to Corrupted Priors}
\label{sec:exp-floor}\label{sec:simulation}
\vspace{-0.5em}
We stress-test Statsformer with deliberately misleading priors to verify the safety floor in 
Theorem~\ref{thm:convex-aggregation}. 
Specifically, we invert each prior score as $s \leftarrow (s+\epsilon)^{-1}$, where $\epsilon$ 
ensures numerical stability, thereby steering learners toward irrelevant features. 
We then compare adversarial-prior Statsformer with the no-prior stacking baseline.
Figure~\ref{tab:inverted_priors_win_ratio} (left) shows that adversarial-prior Statsformer has 
win ratios near the no-prior baseline across datasets. 
Across all datasets, when priors are systematically corrupted, Statsformer reverts to stacking-level performance 
rather than failing catastrophically, suggesting that the gains in Section~\ref{result_main} come 
from informative priors rather than the injection mechanism alone.
\vspace{-0.5em}
\subsection{The Ceiling: Headroom From Informative Priors}
\label{sec:exp-ceiling}
We probe Statsformer's upside with oracle-prior simulations that separate prior quality from 
current LLM capability. 
Across underdetermined ($n=100,p=1000$; $n=300,p=500$; $n=50,p=500$) and well-determined 
($n=500,p=30$) regimes, we supply a perfectly informative feature-relevance prior from the 
ground-truth coefficient structure and compare oracle-prior Statsformer with no-prior stacking 
over $100$ task instantiations per regime (See Appendix \ref{appdx:sim}).
Figure~\ref{tab:inverted_priors_win_ratio} shows that oracle priors yield positive AUROC and accuracy 
gains in underdetermined regimes, with median AUROC gains exceeding $0.1$ in the 
highest-dimensional setting. 
As expected, the gains shrink in the well-determined regime, where sufficient data reduce the 
marginal value of prior information.\footnote{This is consistent with asymptotic efficiency under 
standard regularity conditions \citep{Vaart_1998}: in the large-$n$ limit, the 
achievable benefit of any prior diminishes.} 
Thus, Statsformer preserves a statistical floor through validated stacking while exposing a 
semantic ceiling that is largest in the high-dimensional, data-scarce regimes where 
LLM-informed learning is most valuable.

\section{Discussion and Conclusion}\label{sec:conclusion}
\vspace{-0.5em}
We introduce \emph{Statsformer}, a validated framework for learning when to trust LLM-derived 
semantic priors in supervised statistical learning. 
By injecting priors through learner-specific mechanisms and calibrating them via out-of-fold 
validation, Statsformer secures a \emph{statistical floor} (an oracle-style guarantee against 
any convex combination of in-library candidates) while leaving a \emph{semantic ceiling} that 
informative priors can reach. 
Across diverse datasets, Statsformer consistently outperforms strong baselines; adversarial 
experiments confirm the safety floor in practice, while oracle-prior simulations show that 
substantial headroom remains in high-dimensional, data-scarce regimes, pointing to further 
gains as foundation models improve. 

Natural extensions include richer base learners, broader notions of semantic prior beyond 
feature-level relevance, and applications to settings without explicit feature semantics. 
More broadly, we view validated semantic-prior integration as a step toward reliable 
LLM-informed learning, where the value of foundation-model knowledge is unlocked by mechanisms 
that determine how much it deserves to influence statistical estimation.


\newpage

\begin{ack}
E.Z., N.S., and F.Z. acknowledge support from the Stanford Graduate Fellowship (SGF) for Sciences and Engineering.
N.S. also acknowledges support from the National Science Foundation Graduate Research Fellowship Program under Grant No. DGE-2146755. M.P. acknowledges support in part by National Science Foundation (NSF) CAREER Award under Grant CCF-2236829; in
part by the U.S. Army Research Office Early Career Award
under Grant W911NF-21-1-0242; in part by the Office of
Naval Research under Grant N00014-24-1-2164. J.B. gratefully acknowledges support from DoD through the grants Air Force Office of Scientific Research under award number FA9550-20-1-0397 and ONR 1398311, also support from NSF via grants 2229012, 2312204, 2403007 is gratefully acknowledged.

In addition, we thank Professor Robert Tibshirani for his insightful perspectives on related work, which helped shape our inspiration for this paper.
\end{ack}

\bibliographystyle{plainnat}
\bibliography{main}

\appendix

\section{Comparison to Prior Work}
\begin{table*}[h!]
\centering
\fontsize{5.9}{7.3}\selectfont
\setlength{\tabcolsep}{-0.05pt}
\renewcommand{\arraystretch}{1.15}

\begin{tabularx}{\textwidth}{@{}
    L{1.5cm}
    C{1.20cm}
    C{1.10cm}
    C{1.45cm}
    L{1.6cm}
    C{1.05cm}
    C{1.55cm}
    C{1.55cm}
    C{1.6cm}
    C{1.45cm}
@{}}
\toprule
&
\multicolumn{3}{c}{\textbf{General predictor structure}} &
\multicolumn{4}{c}{\textbf{LLM-specific augmentation \& robustness}} &
\multicolumn{2}{c}{\textbf{Runtime class}} \\
\cmidrule(lr){2-4} \cmidrule(lr){5-8} \cmidrule(lr){9-10}

\textbf{Method} &
\makecell[c]{\textbf{End-to-end}\\\textbf{supervised}\\\textbf{predictor?}} &
\makecell[c]{\textbf{Has}\\\textbf{stacking?}} &
\makecell[c]{\textbf{Heterogeneous}\\\textbf{learner}\\\textbf{library?}} &
\makecell[c]{\textbf{LLM role}} &
\makecell[c]{\textbf{LLM prior}\\\textbf{injection?}} &
\makecell[c]{\textbf{Empirical}\\\textbf{fallback under}\\\textbf{bad LLM}\\\textbf{signal?}} &
\makecell[c]{\textbf{Formal floor}\\\textbf{under bad}\\\textbf{LLM signal?}} &
\makecell[c]{\textbf{Base}\\\textbf{compute}} &
\makecell[c]{\textbf{LLM}\\\textbf{cost}} \\
\midrule

\textbf{XGBoost} &
\textbf{Yes} &
No &
No &
None &
No &
--- &
--- &
Single learner &
None \\

\textbf{LightGBM} &
\textbf{Yes} &
No &
No &
None &
No &
--- &
--- &
Single learner &
None \\

\textbf{Random Forests} &
\textbf{Yes} &
No &
No &
None &
No &
--- &
--- &
Single learner &
None \\

\makecell[l]{\textbf{Super Learner}\\\textbf{(Statsformer-}\\\textbf{no-prior)}} &
\textbf{Yes} &
\textbf{Yes} &
\textbf{Yes} &
None &
No &
--- &
--- &
Ensemble &
None \\

\midrule

\textbf{AutoGluon} &
\textbf{Yes} &
\textbf{Yes} &
\textbf{Yes} &
None &
No &
--- &
--- &
Ensemble &
None \\

\makecell[l]{\textbf{AutoML}\\\textbf{-Agent}} &
\textbf{Yes} &
Task-dependent &
\textbf{Yes} &
Workflow orchestration &
No &
No &
No &
\makecell[c]{Agentic\\workflow} &
\makecell[c]{Workflow\\loop} \\

\midrule

\textbf{LMPriors} &
No &
No &
--- &
Knowledge source &
Indirect &
No &
No &
--- &
Prior query \\

\textbf{LLM-Select} &
No &
No &
--- &
Knowledge source &
Indirect &
No &
No &
--- &
Prior query \\

\textbf{FeatLLM} &
No &
\textbf{Yes} &
--- &
Knowledge source &
Indirect &
Limited &
No &
--- &
Prior query loop \\

\textbf{LLM-Lasso} &
\textbf{Yes} &
No &
No &
Knowledge source &
\textbf{Yes} &
\textbf{Yes} &
No &
Single learner &
Prior query \\

\midrule
\addlinespace[0.45em]

\textbf{Statsformer} &
\textbf{Yes} &
\textbf{Yes} &
\textbf{Yes} &
Knowledge source &
\textbf{Yes} &
\textbf{Yes} &
\textbf{Yes} &
Ensemble &
Prior query \\[0.35em]

\bottomrule
\end{tabularx}

\caption{Comparison across method families.}
\label{tab:method_comparison}
\end{table*}

Table \ref{tab:method_comparison} shows a comparison of \textit{Statsformer} to methods spanning classical statistical baselines, AutoML methods, and LLM-assisted feature methods. Columns are organized into three blocks: \emph{general predictor structure}, \emph{LLM-specific augmentation and robustness}, and \emph{runtime class}. In particular, the runtime block separates two orthogonal sources of cost: \emph{base compute} (single learner, ensemble, or agentic workflow) and \emph{LLM cost} (none, one-time prior query, repeated prior-query loop, or workflow loop). This organization clarifies the appropriate comparison classes: XGBoost, LightGBM, and Random Forests are lower-bound single-model references; Super Learner and AutoGluon are closer ensemble peers; LMPriors, LLM-Select, and FeatLLM use LLMs upstream for feature or pipeline augmentation rather than as end-to-end supervised predictors; LLM-Lasso is the closest single-family semantic-prior baseline; and Statsformer uniquely combines end-to-end prediction, heterogeneous ensembling, direct semantic prior injection, empirical fallback, and a formal downstream performance floor under severe corruption of the LLM-derived signal.

\begin{table*}[h!]
\centering
\fontsize{5.8}{7.1}\selectfont
\setlength{\tabcolsep}{2.4pt}
\renewcommand{\arraystretch}{1.18}

\begin{tabularx}{\textwidth}{@{}
    L{1.72cm}
    L{2.28cm}
    C{1.18cm}
    L{1.95cm}
    C{1.18cm}
    C{1.18cm}
    C{1.28cm}
    C{1.85cm}
@{}}
\toprule
&
\multicolumn{3}{c}{\textbf{Pipeline role}} &
\multicolumn{3}{c}{\textbf{Adapter formalization}} &
\multicolumn{1}{c}{\textbf{Performance guarantees}} \\
\cmidrule(lr){2-4} \cmidrule(lr){5-7} \cmidrule(lr){8-8}

\textbf{Method} &
\makecell[c]{\textbf{How LLM signal}\\\textbf{enters learning}} &
\makecell[c]{\textbf{End-to-end}\\\textbf{supervised}\\\textbf{predictor?}} &
\makecell[c]{\textbf{Base model}\\\textbf{scope}} &
\makecell[c]{\textbf{Multiple}\\\textbf{adapter}\\\textbf{modes$^\ast$?}} &
\makecell[c]{\textbf{Unified}\\\textbf{adapter}\\\textbf{family?}} &
\makecell[c]{\textbf{Theory for}\\\textbf{adapter}\\\textbf{layer?}} &
\makecell[c]{\textbf{Formal downstream}\\\textbf{floor under}\\\textbf{bad priors?}} \\
\midrule

\textbf{LMPriors} &
Upstream feature filtering &
Limited &
Downstream-method dependent &
No &
No &
Limited &
No \\

\textbf{LLM-Select} &
Upstream feature filtering &
No &
Downstream-method dependent &
No &
No &
No &
No \\

\textbf{FeatLLM} &
Upstream feature engineering &
No &
Downstream-method dependent &
No &
No &
No &
No \\

\textbf{LLM-Lasso} &
Single-family embedded penalty adapter &
\textbf{Yes} &
Single learner family (Lasso) &
No &
Limited &
No &
No \\

\addlinespace[0.22em]
\midrule
\addlinespace[0.18em]

\textbf{Statsformer} &
Learner-level adapter family (penalty, scaling, weighting, sampling) &
\textbf{Yes} &
\textbf{Heterogeneous learner library} &
\textbf{Yes} &
\textbf{Yes} &
\textbf{Yes} &
\textbf{Yes} \\[0.12em]

\bottomrule
\end{tabularx}
\caption{Close comparison among LLM-informed statistical learning methods.}
\label{tab:close_peer_comparison}
\end{table*}

Table \ref{tab:close_peer_comparison} provides a detailed comparison among LLM-informed statistical learning methods. The key distinction is not only whether a method uses LLM-derived signal, but \emph{how} that signal enters the statistical pipeline. Columns are organized into three blocks: \emph{pipeline role}, \emph{adapter formalization}, and \emph{performance guarantees}. The table separates five progressively stronger properties: whether the LLM signal acts through upstream filtering/engineering versus a learner-level adapter; whether multiple adapter modes\footnote{Multiple adapter modes = distinct learner-level mechanisms by which LLM-derived signal enters fitting, e.g., penalty modulation, feature-level scaling, instance weighting, or subsampling-based variants.} are supported; whether these modes are unified by a reusable formal adapter family; whether that adapter layer itself is theoretically motivated and analyzed; and whether the overall predictor enjoys a formal downstream floor under bad LLM-derived signal. Here, ``Limited" denotes either task-specific formalization (e.g., feature-relevance scoring in LMPriors) or a single-family tunable adapter (e.g., penalty weighting in LLM-Lasso), rather than a reusable learner-level adapter family spanning different learner types. Relative to the closest direct supervised baseline, LLM-Lasso, Statsformer is not merely stacking: it introduces a validated learner-level adapter family that already instantiates penalty modulation, feature scaling, instance weighting, and sampling-based variants across a heterogeneous learner library, provides theory for that adapter layer, and yields a formal downstream floor under severe prior corruption.

\section{Deferred Proofs}\label{appdx:proof}
Throughout this section, we follow the notation and setup of Section~\ref{sec:theory}. For readability, we restate the relevant risk functionals; wherever applicable, these definitions are identical to those in Section~\ref{sec:theory}.

\subsection{Proof of Theorem \ref{thm:convex-aggregation}}\label{appdx-subsec:oracle}

\begin{proof}

To start, fix a fold $k$ and condition on the training data used to construct
$\{f_l^{(-k)}\}_{l=1}^L$.
Define the prediction class on this fold as the convex hull of the base predictors:
\[
\mathcal{F}_k
:=
\Big\{
x \mapsto f_\pi^{(-k)}(x)
=
\sum_{l=1}^L \pi_l f_l^{(-k)}(x)
\;:\;
\pi\in\Delta^{L-1}
\Big\}.
\]
By assumption, each $f_l^{(-k)}$ satisfies $|f_l^{(-k)}(x)|\le B$,
hence $\mathcal{F}_k$ is uniformly bounded in norm by $B$. 

Let $\mathcal{G}_k := \ell \circ \mathcal{F}_k$ denote the induced loss class
\[
\mathcal{G}_k
=
\Big\{
(x,y)\mapsto \ell(f(x),y)
:\;
f\in\mathcal{F}_k
\Big\}.
\]
Since $\ell(\cdot,y)$ is $L_\ell$-Lipschitz, the vector-contraction inequality for Rademacher complexity (see Theorem 3, Corollary 4 in \citep{maurer2016vectorcontractioninequalityrademachercomplexities} ) yields
\[
\mathfrak{R}_{n_k}(\mathcal{G}_k)
\;\le\;
\sqrt{2}L_\ell\,\mathfrak{R}_{n_k}(\mathcal{F}_k).
\]

Next, define the population and empirical risks restricted to the validation data in fold $k$ following the naming convention as in Section \ref{sec:theory}:
\[
R_k(\pi) = \mathbb{E}\left[\ell\left(\sum_{l=1}^L \pi_lf^{(-k)}_l(X),Y\right)\right], \qquad \hat{R}_k(\pi) = \frac{1}{n_k}\sum_{i \in I_k} \ell\left(\sum_{l=1}^L \pi_lf^{(-k)}_l(x_i),y_i\right),
\]
where the expectation in population risk is taken conditional on the training data of fold $k$.
To apply the standard generalization bound, we assume the loss is bounded by some constant $M$
\footnote{
Since $\ell(\cdot,y)$ is $L_\ell$-Lipschitz and $\ell_{\max}\coloneqq \sup_{y\in\mathcal Y}|\ell(0,y)|<\infty$, we have for any $u$ and $y$,
$
|\ell(u,y)| \le |\ell(u,y)-\ell(0,y)| + |\ell(0,y)|
\le L_\ell \|u\| + \ell_{\max}.
$
Moreover, the bounded-prediction assumption $\|f_l^{(-k)}(x)\|\le B$ for all $k,l,x$ implies that any convex combination $u\in\mathcal F_k$ also satisfies $\|u(x)\|\le B$. Hence $|\ell(u,y)|\le L_\ell B+\ell_{\max}$ uniformly over $u\in\mathcal F_k$ and $y\in\mathcal Y$, so we may take $M \le L_\ell B+\ell_{\max}=O(L_\ell B+\ell_{\max})$.
}.
By applying the standard Rademacher generalization bound (see Theorem~3.3 in \cite{mohri2018foundations}) to the scaled loss class $\{\frac{1}{M}g: g \in \mathcal{G}_k\}$ and rescaling by $M$, we obtain a uniform bound. Specifically, for any $\delta_k\in(0,1)$,
with probability at least $1-\delta_k$ over the validation fold,
\begin{equation}
\sup_{\pi\in\Delta^{L-1}}
\Big|R_k(\pi)-\hat R_k(\pi)\Big|
\;\le\;
2\,\mathfrak{R}_{n_k}(\mathcal{G}_k)
\;+\;
M\sqrt{\frac{\log(2/\delta_k)}{2n_k}}.
\label{eq:fold-rad}
\end{equation}

Since $\mathcal{F}_k$ is the convex hull of the finite class
$\mathcal{F}_{k,0}:=\{f_l^{(-k)}\}_{l=1}^L$,
and Rademacher complexity is invariant under taking the convex hull,
\[
\mathfrak{R}_{n_k}(\mathcal{F}_k)
=
\mathfrak{R}_{n_k}(\mathcal{F}_{k,0}).
\]
Moreover, since $|f_l^{(-k)}(x)|\le B$ for all $\ell$ and $x$,
Massart’s finite class lemma implies
\[
\mathfrak{R}_{n_k}(\mathcal{F}_{k,0})
\;\le\;
B\,\sqrt{\frac{2\log L}{n_k}}.
\]
Combining this with the contraction inequality from above yields
\[
\mathfrak{R}_{n_k}(\mathcal{G}_k)
\;\le\;
\sqrt{2}L_\ell B\,\sqrt{\frac{2\log L}{n_k}}.
\]

Substituting the complexity bound into \eqref{eq:fold-rad}, we obtain the following uniform deviation bound on fold $k$:
for any $\delta_k\in(0,1)$, with probability at least $1-\delta_k$,
\begin{equation}
\sup_{\pi\in\Delta^{L-1}}
\Big|R_k(\pi)-\hat R_k(\pi)\Big|
\;\le\;
2 \sqrt{2} L_\ell B \sqrt{\frac{2\log L}{n_k}}
\;+\;
M\sqrt{\frac{\log(2/\delta_k)}{2n_k}}.
\label{eq:fold-unif}
\end{equation}

Set $\delta_k=\delta/K$ and let $E_k$ denote the event that \eqref{eq:fold-unif}
holds on fold $k$. By the union bound, the probability that all folds satisfy the bound simultaneously is:
\[
\Pr\!\Big(\bigcap_{k=1}^K E_k\Big)
\;\ge\;
1-\sum_{k=1}^K \delta_k
\;=\;
1-\delta.
\]
On the event $\bigcap_{k=1}^K E_k$, for all $\pi\in\Delta^{L-1}$ and all $k$,
\[
R_k(\pi)\le \hat R_k(\pi)+\varepsilon_k,
\qquad
\hat R_k(\pi)\le R_k(\pi)+\varepsilon_k,
\]
where
\[
\varepsilon_k
:=
2 \sqrt{2} L_\ell B \sqrt{\frac{2\log L}{n_k}}
+
M\sqrt{\frac{\log(2K/\delta)}{2n_k}}.
\]

Define the sample-size weighted cross-fitted risks
\[
\hat R_{\mathrm{CF}}(\pi)
:=
\sum_{k=1}^K \frac{n_k}{n}\,\hat R_k(\pi),
\qquad
R_{\mathrm{CF}}(\pi)
:=
\sum_{k=1}^K \frac{n_k}{n}\,R_k(\pi).
\]
Averaging the fold inequalities yields, uniformly over $\pi\in\Delta^{L-1}$,
\[
R_{\mathrm{CF}}(\pi)\le \hat R_{\mathrm{CF}}(\pi)+\varepsilon,
\qquad
\hat R_{\mathrm{CF}}(\pi)\le R_{\mathrm{CF}}(\pi)+\varepsilon,
\]
where
\[
\varepsilon
:=
\sum_{k=1}^K \frac{n_k}{n}\,\varepsilon_k.
\]

Under the assumption of approximately balanced folds ($n_k \approx n/K$), we have $\sum_{k=1}^K \frac{n_k}{n}\frac{1}{\sqrt{n_k}} \approx \sqrt{K/n}.$ Absorbing the loss bound $M$, and the vector-contraction constants into a constant $C_1>0$, we obtain:

\[
\varepsilon
\;\le\;
C_1\,L_\ell B\,\sqrt{\frac{K\log(LK/\delta)}{n}}.
\]

Since $\hat\pi$ minimizes $\hat R_{\mathrm{CF}}$ over the simplex, we use the standard oracle decomposition:
\[
R_{\mathrm{CF}}(\hat\pi)
\;\le\;
\hat R_{\mathrm{CF}}(\hat\pi) + \varepsilon
\;\le\;
\hat R_{\mathrm{CF}}(\pi^*) + \varepsilon
\;\le\;
R_{\mathrm{CF}}(\pi^*) + 2\varepsilon,
\]
where $\pi^*$ is any reference vector in $\Delta^{L-1}$. Thus:
\[
R_{\mathrm{CF}}(\hat\pi)
\;\le\;
\inf_{\pi\in\Delta^{L-1}} R_{\mathrm{CF}}(\pi)
+
2\varepsilon.
\]

Define an auxiliary random index $k^\star\in\{1,\dots,K\}$, independent of $\mathcal D$ and $(X,Y)$, with
$\Pr(k^\star=k)=n_k/n$. Consider the randomized cross-fitted aggregate
\[
\tilde f_{\mathrm{SF}}(X)
\;:=\;
\sum_{l=1}^L \hat\pi_l\, f_l^{(-k^\star)}(X).
\]
Then, by the law of total expectation (equivalently, Lemma~\ref{lemma:risk_id}),
\[
R(\tilde f_{\mathrm{SF}})
\;=\;
\sum_{k=1}^K \frac{n_k}{n}\,
\mathbb{E}\!\left[\ell\!\left(\sum_{l=1}^L \hat\pi_l f_l^{(-k)}(X),Y\right)\right]
\;=\;
\sum_{k=1}^K \frac{n_k}{n}\, R_k(\hat\pi)
\;=\;
R_{\mathrm{CF}}(\hat\pi).
\]
Combining this identity with the cross-validation oracle inequality above yields
\[
R(\tilde f_{\mathrm{SF}})
\;\le\;
\inf_{\pi\in\Delta^{L-1}} R_{\mathrm{CF}}(\pi)
\;+\;
C\,L_\ell B\,\sqrt{\frac{K\log(LK/\delta)}{n}},
\]
for a universal constant $C>0$.
If the candidate dictionary contains a prior-free subset
$\mathcal{F}_{\mathrm{null}}$, then restricting the infimum to
$\pi$ supported on $\mathcal{F}_{\mathrm{null}}$ yields the stated
\emph{no-worse-than-null} guarantee.
\end{proof}


\subsection{Proof of Corollary~\ref{cor:refit}}\label{appdx:refit-proof}

\begin{proof}
Recall that Statsformer selects
\[
\hat\pi \in \argmin_{\pi\in\Delta^{L-1}} \hat R_{\mathrm{CV}}(\pi),
\qquad
\hat f_{\mathrm{SF}}(x) := \sum_{l=1}^L \hat\pi_l \hat f_l(x),
\]
and define the corresponding cross-fitted population risk functional
\[
R_{\mathrm{CF}}(\pi)
:=
\sum_{k=1}^K \frac{n_k}{n}\,
\mathbb{E}\!\left[\ell\!\big(f_\pi^{(-k)}(X),Y\big)\right],
\qquad
f_\pi^{(-k)}(x):=\sum_{l=1}^L \pi_l f_l^{(-k)}(x).
\]

By definition of $\Delta_{\mathrm{refit}}$ and since the supremum ranges over all $\pi$,
in particular over $\pi=\hat\pi$, we have
\begin{equation}
R(\hat f_{\mathrm{SF}})
=
R\!\left(\sum_{l=1}^L \hat\pi_l \hat f_l\right)
\le
R_{\mathrm{CF}}(\hat\pi) + \Delta_{\mathrm{refit}}.
\label{eq:refit-upper-bridge}
\end{equation}
Similarly, for every $\pi\in\Delta^{L-1}$,
\[
R_{\mathrm{CF}}(\pi)
\le
R\!\left(\sum_{l=1}^L \pi_l \hat f_l\right)
+
\Delta_{\mathrm{refit}},
\]
and hence
\begin{equation}
\inf_{\pi\in\Delta^{L-1}} R_{\mathrm{CF}}(\pi)
\le
\inf_{\pi\in\Delta^{L-1}} R\!\left(\sum_{l=1}^L \pi_l \hat f_l\right)
+
\Delta_{\mathrm{refit}}.
\label{eq:refit-lower-bridge}
\end{equation}

Now apply Theorem~\ref{thm:convex-aggregation}, which gives with probability at least $1-\delta$,
\begin{equation}
R_{\mathrm{CF}}(\hat\pi)
\le
\inf_{\pi\in\Delta^{L-1}} R_{\mathrm{CF}}(\pi)
+
C\,L_\ell B\,\sqrt{\frac{K\log(LK/\delta)}{n}}.
\label{eq:thm-cf-risk}
\end{equation}
Combining \eqref{eq:refit-upper-bridge}, \eqref{eq:refit-lower-bridge}, and \eqref{eq:thm-cf-risk} yields
\[
R(\hat f_{\mathrm{SF}})
\le
\inf_{\pi\in\Delta^{L-1}} R\!\left(\sum_{l=1}^L \pi_l \hat f_l\right)
+
C\,L_\ell B\,\sqrt{\frac{K\log(LK/\delta)}{n}}
+
2\Delta_{\mathrm{refit}}.
\]
Finally, under the assumption $\Delta_{\mathrm{refit}}\le \varepsilon_n$, we obtain \eqref{eq:oracle-refit}.
\end{proof}

\subsection{Sufficient conditions for \texorpdfstring{$\Delta_{\mathrm{refit}}\to 0$}{refit gap vanishing}}
\label{appdx:refit-suff}

We provide sufficient conditions under which the refit gap
\begin{align}
\Delta_{\mathrm{refit}}
&:=
\sup_{\pi\in\Delta^{L-1}} |R_{\mathrm{refit}}(\pi) - R_{\mathrm{CF}}(\pi)| \\
&=
\sup_{\pi\in\Delta^{L-1}}
\left|
\mathbb{E}\!\left[\ell\!\left(\sum_{l=1}^L \pi_l \hat f_l(X),Y\right)\right]
-
\sum_{k=1}^K \frac{n_k}{n}\,
\mathbb{E}\!\left[\ell\!\left(\sum_{l=1}^L \pi_l f_l^{(-k)}(X),Y\right)\right]
\right|
\end{align}
vanishes as the sample size $n$ grows.

For each learner configuration $l$, let $\hat f_l$ denote the predictor trained on the full dataset of size $n$, and let $f_l^{(-k)}$ denote the predictor trained on the dataset with fold $k$ removed (size $n-n_k$).

\paragraph{Stability vs. Consistency.}
Classical algorithmic stability arguments are well suited to Leave-One-Out cross-validation, where removing a single observation produces vanishing perturbations ($O(1/n)$) as $n\to\infty$\footnote{Statsformer applies equally to leave-one-out cross-validation; we use $K$-fold in experiments for computational efficiency, but LOO admits stronger perturbation-based guarantees in certain settings (e.g., for the Lasso \citep{homrighausen2013leaveoneoutcrossvalidationriskconsistent}).
}. In contrast, for $K$-fold cross-validation with fixed $K$, a constant fraction of the data is removed in each fold. Consequently, the distance $\|\hat f_l - f_l^{(-k)}\|$ does not vanish due to stability alone. Instead, we rely on the \emph{consistency} of the base learners: if both $\hat f_l$ and $f_l^{(-k)}$ converge to the same fixed population limit $f^*_l$, their difference must vanish.

\begin{assumption}[Lipschitz loss in prediction]
\label{ass:lipschitz-loss}
There exists $L_\ell>0$ such that for all $y\in\mathcal{Y}$ and all predictions $u,v$,
\[
|\ell(u,y)-\ell(v,y)| \le L_\ell \|u-v\|.\footnote{We note that Assumption~\ref{ass:lipschitz-loss} coincides with the Lipschitz condition imposed in Theorem~\ref{thm:convex-aggregation}.}
\]
\end{assumption}

\begin{assumption}[Concentration control for base learners]
\label{ass:consistency}
For each learner configuration $l$, let $f^*_l$ denote the population-limit predictor
associated with configuration $l$. We assume that for any $\delta\in(0,1)$ and any
sample size $m$, there exists a function $\epsilon_l(m,\delta)$ such that
\[
\mathbb{P}\!\left(\|\hat f_{l,m}-f^*_l\|_2 \le \epsilon_l(m,\delta)\right) \ge 1-\delta,
\]
where $\|g\|_2 := \big(\mathbb{E}[\|g(X)\|_2^2]\big)^{1/2}$ denotes the $\mathcal{L}^2(P_X)$ norm
of prediction error, and for each fixed $\delta$, $\epsilon_l(m,\delta)\to 0$ as $m\to\infty$.
\end{assumption}


\begin{lemma}[Refit gap bound from consistency]
\label{lem:refit-gap}
Under Assumptions~\ref{ass:lipschitz-loss} and~\ref{ass:consistency}, for any
$\delta\in(0,1)$, with probability at least $1-\delta$,
\[
\Delta_{\mathrm{refit}}
\le
L_\ell
\max_{l\in\{1,\dots,L\}}
\max_{k\in\{1,\dots,K\}}
\Big(
\epsilon_l(n,\delta') + \epsilon_l(n-n_k,\delta')
\Big),
\]
where $\delta'=\delta/(2LK)$. In particular, if $\epsilon_l(m,\delta)=O(m^{-1/2})$
for all $l$ (for fixed $\delta$), then $\Delta_{\mathrm{refit}}=O(n^{-1/2})$ with
probability at least $1-\delta$.
\end{lemma}

\begin{proof}
Fix $\delta\in(0,1)$ and set $\delta'=\delta/(2LK)$. For each $l$, define the event
\[
A_l := \left\{ \|\hat f_{l,n}-f_l^*\|_2 \le \epsilon_l(n,\delta') \right\},
\]
and for each $l,k$ define
\[
B_{l,k} := \left\{ \|f_l^{(-k)}-f_l^*\|_2 \le \epsilon_l(n-n_k,\delta') \right\}.
\]
By Assumption~\ref{ass:consistency}, $\Pr(A_l)\ge 1-\delta'$ and
$\Pr(B_{l,k})\ge 1-\delta'$. A union bound over all $l$ and $(l,k)$ gives
\[
\Pr\!\Big(\bigcap_{l=1}^L A_l \;\cap\; \bigcap_{l=1}^L \bigcap_{k=1}^K B_{l,k}\Big)
\;\ge\; 1 - (L + LK)\delta' \;\ge\; 1-\delta.
\]
On this event, for any $\pi\in\Delta^{L-1}$, by Lipschitzness and the triangle
inequality,
\begin{align*}
\Delta_{\mathrm{refit}}
&=
\sup_{\pi}
\sum_{k=1}^K \frac{n_k}{n}
\Big|
\mathbb{E}\big[\ell(\sum_l \pi_l \hat f_l(X),Y)\big]
-
\mathbb{E}\big[\ell(\sum_l \pi_l f_l^{(-k)}(X),Y)\big]
\Big| \\
&\le
\sup_{\pi}
\sum_{k=1}^K \frac{n_k}{n} L_\ell
\mathbb{E}\left[
\left\|\sum_{l=1}^L \pi_l\big(\hat f_l(X)-f_l^{(-k)}(X)\big)\right\|
\right].
\end{align*}
Using Jensen’s inequality,
\[
\mathbb{E}\left\|
\sum_{l=1}^L \pi_l\big(\hat f_l(X)-f_l^{(-k)}(X)\big)
\right\|
\le
\left\|
\sum_{l=1}^L \pi_l\big(\hat f_l-f_l^{(-k)}\big)
\right\|_2.
\]
By convexity of the norm and the triangle inequality relative to $f_l^*$,
\[
\left\|
\sum_{l=1}^L \pi_l(\hat f_l-f_l^{(-k)})
\right\|_2
\le
\sum_{l=1}^L \pi_l\big(\|\hat f_l-f_l^*\|_2 + \|f_l^{(-k)}-f_l^*\|_2\big).
\]
On the event above, each term is bounded by
$\epsilon_l(n,\delta') + \epsilon_l(n-n_k,\delta')$, so
\[
\Delta_{\mathrm{refit}}
\le
L_\ell
\max_{l,k}
\Big(
\epsilon_l(n,\delta') + \epsilon_l(n-n_k,\delta')
\Big).
\]
This proves the high‑probability bound. The rate statement follows immediately
when $\epsilon_l(m,\delta)=O(m^{-1/2})$ for fixed $\delta$.
\end{proof}

\subsubsection{Concrete sufficient condition: strongly convex regularized ERM}

Assumption~\ref{ass:consistency} holds for standard parametric learners trained by strongly convex regularized empirical risk minimization.

\begin{proposition}[Consistency of strongly convex ERM]
\label{prop:erm-consistency}
Consider a parametric class $\{f_\theta:\theta\in\Theta\}$ and a regularized ERM estimator
\[
\hat\theta_n
\in
\argmin_{\theta\in\Theta}
\left\{
\frac{1}{n}\sum_{i=1}^n \phi(\theta;x_i,y_i)
+
\lambda r(\theta)
\right\},
\]
where $\phi(\cdot;x,y)$ is convex and $G$-Lipschitz \emph{in $\theta$}, and assume that $r(\cdot)$ is $\mu_r$-strongly convex, thus the population objective
\[
F(\theta):=\mathbb{E}[\phi(\theta;X,Y)] + \lambda r(\theta)
\]
is $\mu$-strongly convex for $\mu=\lambda \mu_r>0$ (see Definition~\ref{def:strong-convex}).
Let $\theta^*=\argmin_{\theta\in\Theta} F(\theta)$ denote the unique population minimizer. Then
\[
\mathbb{E}\!\left[\|\hat\theta_n-\theta^*\|_2\right]
\le
\frac{C\,G}{\mu\sqrt{n}},
\]
for a universal constant $C>0$. Consequently, Assumption~\ref{ass:consistency} holds with
$\epsilon_l(n)=O(n^{-1/2})$.
\end{proposition}

\begin{proof}[Proof sketch]
Let $F(\theta):=\mathbb{E}[\phi(\theta;X,Y)] + \lambda r(\theta)$ and
$\hat F_n(\theta):=\frac{1}{n}\sum_{i=1}^n \phi(\theta;x_i,y_i)+\lambda r(\theta)$.
We have that $F$ is $\mu$-strongly convex with unique minimizer $\theta^*$.
Strong convexity implies
$
F(\hat\theta_n)-F(\theta^*) \ge \frac{\mu}{2}\|\hat\theta_n-\theta^*\|_2^2.
$
Moreover, for regularized ERM with $G$-Lipschitz loss (in $\theta$) and $\mu$-strongly convex
objective, uniform stability yields the excess risk bound
$
\mathbb{E}\!\left[F(\hat\theta_n)-F(\theta^*)\right] \le \frac{c\,G^2}{\mu n}
$
for a universal constant $c$ (see \cite{bousquet2002stability}).
Combining the two displays gives
$\mathbb{E}\|\hat\theta_n-\theta^*\|_2^2 \le \frac{2cG^2}{\mu^2 n}$, hence
$
\mathbb{E}\|\hat\theta_n-\theta^*\|_2 \le \frac{\sqrt{2c}\,G}{\mu\sqrt{n}}.
$
\end{proof}


\paragraph{Remark on regularization.}

Statsformer operates over a finite dictionary of learner configurations with fixed
hyperparameters. In particular, regularization parameters such as $\lambda$ are fixed
within each configuration and do not decay with $n$, ensuring that the consistency rate
in Proposition~\ref{prop:erm-consistency} applies uniformly across the ensemble. This setting covers standard strongly convex learners such as ridge regression and
regularized generalized linear models, as well as practical AutoML-style pipelines with
a finite hyperparameter grid.

\paragraph{Remark on data-dependent regularization (Lasso).}
For estimators such as the Lasso, the tuning parameter may depend on sample size,
either through a deterministic schedule $\lambda_m$ or data-adaptively via cross-validation.
As a result, classical stability arguments based on uniform strong convexity are not the most direct tool
for analyzing $K$-fold refitting.
In our framework, this is handled by treating the configuration $l$ as including the tuning rule for $\lambda$,
and letting $f_l^*$ denote the corresponding population-limit predictor.
Under standard regularity conditions, Lasso predictors tuned by leave-one-out cross-validation are risk-consistent,
and the leave-one-out and full-sample fits are asymptotically equivalent
\citep{homrighausen2013leaveoneoutcrossvalidationriskconsistent}.
Consequently, Assumption~\ref{ass:consistency} holds for the Lasso in prediction, and by
Lemma~\ref{lem:refit-gap} we obtain $\Delta_{\mathrm{refit}}\to 0$
(with rates following once $\epsilon_l(m)$ is specified).

\paragraph{Example: Linear predictors.}
For linear predictors $f_\theta(x)=\langle\theta,x\rangle$ with $\|x\|_2\le B_x$,
\[
\|\hat f_{l,n}-f^*_l\|_2
=
\Big(\mathbb{E}[|\langle \hat\theta_n-\theta^*, X\rangle|^2]\Big)^{1/2}
\le
B_x\|\hat\theta_n-\theta^*\|_2
= O(n^{-1/2}),
\]
verifying Assumption~\ref{ass:consistency}.

\paragraph{Remark on Tree-Based Methods.}
While Proposition~\ref{prop:erm-consistency} focuses on convex parametric models,
Assumption~\ref{ass:consistency} (consistency in $\ell_2$) can also be satisfied by
several nonparametric learners in the Statsformer dictionary, under standard
conditions and with appropriate regularization or tuning.
\begin{itemize}
    \item \emph{Random Forests:}
Under additive regression models, the (infinite) Breiman random forest is
$\mathcal{L}^2$-consistent for the conditional mean when tree complexity is suitably
controlled (e.g., the number of leaves grows appropriately with the subsample
size), and for fully grown trees under additional technical conditions,
including a vanishing subsampling rate \citep{scornet2015consistency}.
    \item \emph{Gradient Boosting:}
Stage-wise boosting can be made consistent via early stopping, often paired with
shrinkage (small step sizes), which implicitly controls the complexity of the
additive model in a manner analogous in spirit to the penalty in Proposition~\ref{prop:erm-consistency} \citep{zhang2005boosting, bartlett2007adaboost}.

\end{itemize}
Thus, although tree-based learners need not satisfy the strong convexity
conditions of Proposition~\ref{prop:erm-consistency}, they can be configured to
satisfy the consistency requirement in Assumption~\ref{ass:consistency},
ensuring the refit gap vanishes under the assumptions of our analysis.

\subsection{Extensions and Relaxations}
\label{subsec:extension}

The oracle guarantees in Theorem~\ref{thm:convex-aggregation} are stated under
boundedness and Lipschitz assumptions that enable a clean Rademacher-based analysis
and isolate the statistical role of validated prior integration.
These assumptions are standard in aggregation theory and can be relaxed in several
directions. We briefly outline the most relevant extensions.

\paragraph{Multiclass classification.}
Theorem~\ref{thm:convex-aggregation} applies directly to $k$-class classification,
where predictors output score or probability vectors in $\mathbb{R}^k$.
For standard losses such as softmax cross-entropy, the loss is Lipschitz on bounded
prediction domains, and the boundedness assumption holds automatically when predictions
lie in the probability simplex or are uniformly clipped.

\paragraph{Unbounded or locally Lipschitz losses.}
The bounded-loss assumption in Theorem~\ref{thm:convex-aggregation} can be relaxed
to cover unbounded losses under standard tail conditions.
In particular, suppose $\ell(\cdot,y)$ is locally Lipschitz and the induced losses
$\ell(f_l^{(-k)}(X),Y)$ are uniformly sub-Gaussian.
Then, by using concentration inequalities for sub-Gaussian variables, one obtains
an oracle inequality for \emph{model selection} (competing with the best single predictor in-dictionary).
See proposition below.

\begin{proposition}[Oracle Model Selection under sub-Gaussian losses]
\label{prop:subg-selection} 
Assume the setup of Theorem~\ref{thm:convex-aggregation}, but relax the boundedness assumption.
Assume instead that for every fold $k$ and every configuration $l \in \{1, \dots, L\}$, conditional on the training data used to construct $f_l^{(-k)}$,
the centered random variable \footnote{Recall that the expectation is conditional on the training data for an independent copy $(X,Y)$ sampled from the population distribution. This assumptions is satisfied, for example, if There exists $\sigma>0$ such that for every predictor $f$ in the (data-dependent) dictionary,
$
\ell(f(X),Y) - \mathbb{E}\!\left[\ell(f(X),Y)\right]
$
is $\sigma^2$-sub-Gaussian.
}
\[
Z_{k,\ell} := \ell(f_l^{(-k)}(X),Y) - \mathbb{E}[\ell(f_l^{(-k)}(X),Y)]
\]
is $\sigma^2$-sub-Gaussian.\footnote{This assumptions is satisfied, for example, if There exists $\sigma>0$ such that for every predictor $f$ in the (data-dependent) dictionary,
$
\ell(f(X),Y) - \mathbb{E}\!\left[\ell(f(X),Y)\right]
$
is $\sigma^2$-sub-Gaussian.
} Let $\tilde f_{\mathrm{Sel}}$ be the selector that minimizes the empirical cross-validation risk over the finite set of configurations $\{1, \dots, L\}$.
Then there exists a universal constant $C>0$ such that for any $\delta\in(0,1)$,
with probability at least $1-\delta$,
\[
R(\tilde f_{\mathrm{Sel}})
\le
\min_{l \in \{1, \dots, L\}}
\sum_{k=1}^K \frac{n_k}{n}\,
\mathbb{E}\!\left[\ell\!\big(f_l^{(-k)}(X),Y\big)\right]
+
C\,\sigma\,
\sqrt{\frac{K \log(LK/\delta)}{n}}.
\]
\end{proposition}

\begin{proof}
Let $\hat l$ be the index selected by minimizing $\hat R_{\mathrm{CV}}(\cdot)$,
and let $l^*$ be the index minimizing the true CV risk $R_{\mathrm{CF}}(\cdot)$.
Note that $R(\tilde f_{\mathrm{Sel}}) = R_{\mathrm{CF}}(\hat l)$.\footnote{For the analysis in this proof, we slightly abuse notation by using the configuration index $l$ as the argument for the risk functions $R_{\mathrm{CF}}(\cdot)$ and $\hat R_{\mathrm{CV}}(\cdot)$ for brevity.}
By adding and subtracting empirical terms, we have
\begin{align}\label{eq:inequality1}
\begin{split}
R_{\mathrm{CF}}(\hat l) - R_{\mathrm{CF}}(l^*)
&=
\big[R_{\mathrm{CF}}(\hat l) - \hat R_{\mathrm{CV}}(\hat l)\big]
+ \big[\hat R_{\mathrm{CV}}(\hat l) - \hat R_{\mathrm{CV}}(l^*)\big]
+ \big[\hat R_{\mathrm{CV}}(l^*) - R_{\mathrm{CF}}(l^*)\big] \\
&\le
\big|R_{\mathrm{CF}}(\hat l) - \hat R_{\mathrm{CV}}(\hat l)\big|
+ \big|\hat R_{\mathrm{CV}}(l^*) - R_{\mathrm{CF}}(l^*)\big| \\
&\le
2 \sup_{l' \in \{1,\dots,L\}}
\big| \hat R_{\mathrm{CV}}(l') - R_{\mathrm{CF}}(l') \big|,
\end{split}
\end{align}
where the first inequality uses the optimality of $\hat l$ for $\hat R_{\mathrm{CV}}$,
which implies $\hat R_{\mathrm{CV}}(\hat l) \le \hat R_{\mathrm{CV}}(l^*)$.

Fix a fold $k$ and an index $l$. Condition on the training data used to construct $f_l^{(-k)}$.
The validation losses $Z_i := \ell(f_l^{(-k)}(x_i), y_i) - \mathbb{E}[\ell]$ for $i \in I_k$ are independent, mean-zero, $\sigma^2$-sub-Gaussian variables.
Recall that the sub-Gaussian parameter $\sigma$ is equivalent to the sub-Gaussian norm $\|\cdot\|_{\psi_2}$ up to an absolute constant, i.e., $\|Z_i\|_{\psi_2} \lesssim \sigma$.
Applying the general Hoeffding inequality (Thm 2.7.3 in \cite{vershynin:hdp2018}) to the sum of these variables:
\[
\Pr\left( \left| \sum_{i \in I_k} Z_i \right| \ge t \right)
\le
2 \exp\left( - \frac{c t^2}{\sum_{i \in I_k} \|Z_i\|_{\psi_2}^2} \right)
\le
2 \exp\left( - \frac{c' t^2}{n_k \sigma^2} \right).
\]
Setting $t = n_k \epsilon$ to convert the sum to an average $\hat R_k(l) - R_k(l)$:
\[
\Pr\left( \left| \hat R_k(l) - R_k(l) \right| \ge \epsilon \right)
\le
2 \exp\left( - \frac{c' (n_k \epsilon)^2}{n_k \sigma^2} \right)
=
2 \exp\left( - \frac{c' n_k \epsilon^2}{\sigma^2} \right).
\]

We apply a union bound over all $L$ models and all $K$ folds. Set the failure probability for a single pair $(k, l)$ to $\delta' = \delta / (LK)$.
Solving the probability bound for $\epsilon$, we find that with probability at least $1-\delta$, simultaneously for all $k, l$:
\[
\left| \hat R_k(l) - R_k(l) \right|
\le
C_1 \sigma \sqrt{\frac{\log(2LK/\delta)}{n_k}},
\]
where $C_1$ is a constant depending only on $c'$.

The global deviation is the weighted sum of fold deviations. Using $n_k \approx n/K$:
\[
\sup_{l} \big| \hat R_{\mathrm{CV}}(l) - R_{\mathrm{CF}}(l) \big|
\le
\sum_{k=1}^K \frac{n_k}{n} C_1 \sigma \sqrt{\frac{\log(2LK/\delta)}{n_k}}
\le
C_2 \sigma \sqrt{\frac{K \log(LK/\delta)}{n}},
\]
where $C_2$ absorbs $C_1$ and numerical factors from the logarithm approximation.
Substituting this bound back into Inequality \eqref{eq:inequality1} yields:
\[
R_{\mathrm{CF}}(\hat l) - R_{\mathrm{CF}}(l^*) \le 2 \left( C_2 \sigma \sqrt{\frac{K \log(LK/\delta)}{n}} \right).
\]
The result follows by defining the final universal constant $C = 2 C_2$.
\end{proof}

\paragraph{Example: squared-error regression.}
Consider the squared loss $\ell(u,y)=(u-y)^2$.
If $|f_l^{(-k)}(x)|\le M$ and $|Y|\le M$ almost surely, then for any $y\in[-M,M]$
the map $u\mapsto \ell(u,y)$ is $4M$-Lipschitz on $[-M,M]$.
Thus Theorem~\ref{thm:convex-aggregation} applies with $L_\ell=4M$ and $B=M$.

More generally, suppose $Y_i$ and $f_l^{(-k)}(X_i)$ are $\sigma^2$-sub-Gaussian
(uniformly over $i,k,l$). By the sub-Gaussian tail bound and a union bound over
$N \asymp nLK$ variables, with probability at least $1-\delta$,
\[
|f_l^{(-k)}(X_i)| \;\vee\; |Y_i|
\le
C\sigma\sqrt{\log(nLK/\delta)}
\qquad \text{for all } i,k,l,
\]
for a universal constant $C$ \cite{vershynin:hdp2018}.
On this event, the squared loss is $L_{\mathrm{loc}}$-Lipschitz in its first argument with
$L_{\mathrm{loc}} = 4C\sigma\sqrt{\log(nLK/\delta)}$.
Applying Theorem~\ref{thm:convex-aggregation} on the localized event and controlling the
complement via a standard truncation argument yields analogous oracle guarantees up to
logarithmic factors; see, e.g., \cite{bousquet2002concentration,catoni2012challenging,mendelson2022learning}.

\subsection{Additional Deferred Definitions and Lemmas}

\begin{definition}[$\lambda$-strong convexity]\label{def:strong-convex}
Let $f:\mathcal{W}\to\mathbb{R}$ be a function and let $\|\cdot\|$ be a norm on $\mathcal{W}$.
We say that $f$ is \emph{$\lambda$-strongly convex} for some $\lambda>0$ if for all
$w,w'\in\mathcal{W}$ and all $t\in[0,1]$,
\[
f(tw+(1-t)w')
\le
t f(w) + (1-t) f(w')
-
\frac{\lambda}{2} t(1-t)\|w-w'\|^2.
\]
\end{definition}

\begin{remark}[First-order characterization]
If $f$ is differentiable, then $f$ is $\lambda$-strongly convex with respect to
$\|\cdot\|$ if and only if for all $w,w'\in\mathcal{W}$,
\[
f(w') \ge f(w)
+ \langle \nabla f(w), w'-w\rangle
+ \frac{\lambda}{2}\|w'-w\|^2.
\]
\end{remark}

\begin{lemma}[Risk Identification]\label{lemma:risk_id}
Let $\hat{\pi}$ be the weight vector obtained via $K$-fold cross-validation. The population risk of the cross-fitted aggregate $\tilde f_{\mathrm{SF}}$, denoted $R(\tilde f_{\mathrm{SF}})$, is identically the population cross-fit risk functional $R_{\mathrm{CF}}(\pi)$ evaluated at $\pi = \hat{\pi}$:
\[
R(\tilde f_{\mathrm{SF}}) = R_{\mathrm{CF}}(\hat{\pi}).
\]
\end{lemma}
\begin{proof}
Let $\mathcal{D} = \{(x_i, y_i)\}_{i=1}^n$ denote the training dataset and let $(X,Y) \sim P$ be a test point independent of $\mathcal{D}$. We define a random index $k^\star$ independent of both $\mathcal{D}$ and $(X,Y)$, such that $P(k^\star = k) = n_k/n$ for $k = 1, \dots, K$. The randomized predictor is defined as $\tilde f_{\mathrm{SF}}(X) = \sum_{l=1}^L \hat\pi_l f_l^{(-k^\star)}(X)$.

We evaluate the population risk of $\tilde f_{\mathrm{SF}}$ by conditioning on the observed data $\mathcal{D}$ and applying the law of total expectation over the auxiliary variable $k^\star$:
\begin{align*}
R(\tilde f_{\mathrm{SF}}) 
&= \mathbb{E}\!\left[ \ell(\tilde f_{\mathrm{SF}}(X), Y) \mid \mathcal{D} \right] \\
&= \mathbb{E}_{k^\star} \!\left[ \mathbb{E}_{(X,Y)} \!\left[ \ell\left( \sum_{l=1}^L \hat\pi_l f_l^{(-k^\star)}(X), Y \right) \Bigg| \mathcal{D}, k^\star \right] \right] \\
&= \sum_{k=1}^K \frac{n_k}{n} \, \mathbb{E}_{(X,Y)} \!\left[ \ell\left( \sum_{l=1}^L \hat\pi_l f_l^{(-k)}(X), Y \right) \Bigg| \mathcal{D} \right].
\end{align*}

Given $\mathcal{D}$, the weight vector $\hat{\pi}$ and the predictors $\{f_l^{(-k)}\}_{l=1}^L$ are fixed constants (i.e., they are $\sigma(\mathcal{D})$-measurable). Because $(X,Y)$ is independent of $\mathcal{D}$, the inner expectation for a fixed $k$ is equivalent to the population risk functional $R_k$ evaluated at the cross-validated weights:
\[
\mathbb{E}_{(X,Y)} \!\left[ \ell\left( \sum_{l=1}^L \hat\pi_l f_l^{(-k)}(X), Y \right) \Bigg| \mathcal{D} \right] = R_k(\hat\pi),
\]
where $R_k(\pi) \coloneqq \mathbb{E}_{(X,Y)\sim P} \!\left[ \ell\left( \sum_l \pi_l f_l^{(-k)}(X), Y \right) \right]$. Substituting this back into the summation, we obtain:
\[
R(\tilde f_{\mathrm{SF}}) = \sum_{k=1}^K \frac{n_k}{n} R_k(\hat\pi) = R_{\mathrm{CF}}(\hat\pi).
\]
This completes the proof.
\end{proof}

\section{Injection Mechanism}
\section{Modes of Prior Injection}\label{appdx:prior-injection}

This appendix provides the formal instantiations of the general prior-injection principle described in the main text. 
These forms are not exhaustive, but they cover a large fraction of commonly used learning architectures and illustrate how the same abstract principle yields concrete, implementable objectives across model families.

\paragraph{Penalty-based injection.}
For linear predictors $b_m(x;\theta_m)=\langle x,\theta_m\rangle$ or generalized linear models, the prior modulates the regularization geometry via feature-specific penalty weights $w_j(\alpha)=\tau_\alpha(v_j)$:
\begin{align*}
\begin{split}
\mathcal{L}_m(\theta_m;\alpha)
&=
\frac{1}{n}\sum_{i=1}^n \ell\big(b_m(x_i;\theta_m),y_i\big)
+ \lambda_m \sum_{j=1}^p w_j(\alpha)\,\phi(\theta_{m,j}),
\end{split}
\end{align*}
where $\phi(\cdot)$ denotes a coordinate-wise regularizer (e.g., $\ell_1$, $\ell_2$, or elastic-net penalties).
The null condition $\alpha=0$ yields uniform penalization and recovers the prior-free learner.

\paragraph{Feature-reweighting injection.}
For nonlinear learners whose training depends on the feature representation, the prior is injected by reweighting input coordinates,
\begin{align}
\label{eq:scale-injection}
\tilde X_{ij}(\alpha) = s_j(\alpha)\,X_{ij},
\qquad
s_j(\alpha)=\tau_\alpha(v_j),
\end{align}
after which training proceeds by minimizing
\[
\mathcal{L}_m(\theta_m;\alpha)
=
\frac{1}{n}\sum_{i=1}^n \ell\big(b_m(\tilde x_i(\alpha);\theta_m),y_i\big)
+
\lambda_m \Omega_m(\theta_m).
\]

This rescaling emphasizes directions associated with large $v_j$ in kernel evaluations or gradient flow, while recovering the original model when $\alpha=0$.

Tree-based methods are typically insensitive to feature scaling but can incorporate feature priors through feature subsampling. Methods such as XGBoost support feature weights directly, while other implementations can inject priors by sampling features in proportion to $s_j(\alpha)$.\footnote{Most tree-based learners perform random feature subsampling at each split; when explicit feature weights are unavailable, this behavior can be emulated by oversampling features according to their prior-induced weights.}

\paragraph{Instance-weight injection.}
For learners that accept observation weights, the feature prior induces sample weights
$\rho_i(\beta)$, leading to the objective
\begin{align*}
\mathcal{L}_m(\theta_m;\beta)
=
\frac{1}{n}\sum_{i=1}^n \rho_i(\beta)\,
\ell\big(b_m(x_i;\theta_m),y_i\big)
+
\lambda_m \Omega_m(\theta_m),
\end{align*}
which concentrates training on samples where LLM-identified features are active. We use the formulation
\[\rho_i(\beta)
=
(1-\beta)+\beta\,\frac{\tilde{\rho}_i-\min_k \tilde{\rho}_k}{\max_k \tilde{\rho}_k-\min_k \tilde{\rho}_k + \varepsilon}, \]
where $\epsilon$ is a small constant, $\beta\in[0,1]$ interpolates between uniform ($\beta=0$) and
prior-aware weight ($\beta=1$), and
\begin{align}
\tilde{\rho}_i
\;=\;
\sum_{j=1}^p v_j\,|x_{ij}|^q,
\qquad q\geq 1,
\end{align}
This aggregates the activation of semantically relevant features in the sample.
\footnote{Appendix~\ref{subsec:kl-instance} derives this formula via a KL projection of the empirical distribution onto a density induced by the prior $V$.}

Across all cases, the external prior $V$ influences learning only through the monotone family $\tau_\alpha$, and every transformation satisfies $\tau_0(v_j)=1$.
Hence, each learner family contains a prior-free baseline, yielding a built-in in-library ``no-worse-than-null'' safeguard up to statistical errors once models are selected via out-of-fold validation.

\section{Principled Monotone Map Families}\label{appdx:monotone_map_families}

Statsformer injects semantic priors into statistical learners through monotone transformations $\tau_\alpha:\cV\to\R_{>0}$ that modulate penalties, feature reweighting (e.g., scaling or subsampling weights), or sample weights. To move beyond ad hoc engineering, we can demand that each family of maps arises from an optimality principle: the transformation should be the unique solution of a probabilistic calibration or divergence minimization problem under the prior information supplied by the LLM. Once the principle is fixed, the resulting map is automatically monotone, satisfies $\tau_0(v)=1$, and inherits a clear semantic interpretation. This section develops two concrete families, which are used in the implementation for Statsformer.

\subsection{Bayesian Interpretation of Penalty Weights}
\label{subsec:map-penalty}

Consider the weighted Lasso estimator
\[
\hat\beta
=
\arg\min_{\beta\in\mathbb{R}^p}
\frac{1}{2\sigma^2}\|Y-X\beta\|_2^2
+
\lambda \sum_{j=1}^p w_j |\beta_j|,
\qquad
w_j=\tau_\alpha(v_j).
\]
This objective corresponds to the MAP estimator under a Gaussian likelihood
$Y\mid\beta\sim\mathcal{N}(X\beta,\sigma^2 I)$
and independent Laplace priors
$p(\beta_j\mid b_j)\propto \exp(-|\beta_j|/b_j)$,
with the identification $w_j = 1/b_j$.

To connect $b_j$ to the LLM-derived prior $v_j$, we adopt a spike-and-slab perspective.
Let
\[
\beta_j\mid\gamma_j \sim (1-\gamma_j)\delta_0 + \gamma_j\,\text{Laplace}(0,b_{\max}),
\qquad
\gamma_j\sim\text{Bernoulli}(\pi_j),
\]
where $\pi_j=\rho(v_j)$ is a monotone calibration of the semantic score.
Following the adaptive Lasso approximation \citep{zou2006adaptive},
the discrete spike is replaced by a continuous penalty with effective scale
\[
b_j = \frac{b_{\max}}{\max(\pi_j,\epsilon)}, \qquad \epsilon>0.
\]
The resulting penalty weight is
\[
w_j = \frac{1}{b_j} = \frac{\max(\pi_j,\epsilon)}{b_{\max}},
\]
which is monotone increasing in $v_j$ whenever $\rho$ is monotone.
Thus, semantically relevant features (large $v_j$) receive smaller penalties.

This derivation provides a Bayesian motivation for the monotone penalty families
used in Statsformer, without requiring exact posterior computation.

\subsection{Instance Weighting via KL Projection}
\label{subsec:kl-instance}

Appendix~\ref{appdx:prior-injection} defines sample weights $\rho_i(\beta)$ by hand. We can recover the same structure from an information projection that aligns the empirical distribution with a prior density over $(X,Y)$ implied by the LLM feature scores.

Let $P_n$ be the empirical measure that assigns mass $1/n$ to each observation $(X_i,Y_i)$.
Assume the LLM provides a nonnegative signal $s_i = \psi(v,X_i)$, such as the feature-importance score
$s_i(V,q) = \sum_{j=1}^p v_j|x_{ij}|^q$.
Define the target tilted distribution
\begin{align}
    Q^\star = \argmin_{Q\ll P_n}
        \left\{
            \mathrm{KL}(Q \Vert P_n)
            \ \text{s.t.}\ \E_Q[\phi(X,Y)] = c(v)
        \right\},
\end{align}
where $\phi$ is a sufficient statistic chosen to emphasize high-prior regions (e.g., $\phi(X,Y)=s(X)$) and $c(v)$ encodes the desired shift derived from $V$.
Standard exponential-family duality implies that $Q^\star$ is an exponential tilt of $P_n$:
\begin{align}
    \frac{dQ^\star}{dP_n}(X_i,Y_i)
    = \frac{\exp(\eta^\top \phi(X_i,Y_i))}{Z(\eta)}, \qquad
    \eta = \eta(v),
\end{align}
with partition function $Z(\eta) = \tfrac{1}{n}\sum_{k=1}^n \exp(\eta^\top \phi(X_k,Y_k))$.
Setting $\phi(X_i,Y_i)=s_i$ and $c(v)=\bar{s} + \beta(\bar{s}_{\text{prior}} - \bar{s})$, where $\bar{s}$ is the empirical mean and $\bar{s}_{\text{prior}}$ the prior-implied target, yields the moment equation
\begin{align}
    \frac{1}{n}\sum_{i=1}^n s_i \frac{\exp(\eta s_i)}{Z(\eta)}
    = (1-\beta)\bar{s} + \beta\, \bar{s}_{\text{prior}},
    \qquad \beta\in[0,1].
\end{align}
Solving for $\eta=\eta(\beta)$ produces weights
\begin{align}
    \rho_i(\beta)
    = \frac{dQ^\star}{dP_n}(X_i,Y_i)
    = \frac{\exp(\eta(\beta) \, s_i)}{Z(\eta(\beta))}.
\end{align}
Because $\eta(\beta)$ increases with $\beta$ (by the implicit function theorem under mild regularity), the map
\begin{align}
    \tau_\beta(v_i)
    = \rho_i(\beta)
\end{align}
is monotone in $\beta$ and in $s_i$, hence monotone in $v$. Linearizing near $\beta=0$ gives
\begin{align}
    \rho_i(\beta)
    \approx 1 + \beta \cdot \frac{s_i - \bar{s}}{\Var_{P_n}(s)},
\end{align}
which matches the affine blend in the main text (after rescaling and clipping to keep weights positive).
Thus the ``engineering'' formula is the first-order approximation of the exponential-family solution. Using the exact tilt instead yields a principled weight with the same computational complexity—up to the one-dimensional root-finding required to determine $\eta(\beta)$.

If the LLM provides a full density $r(X)$ rather than feature scores, we simply replace $s_i$ by $\log r(X_i)$ and set $c(v)$ to match the expected log-likelihood under $Q^\star$. The optimal weights become
\begin{align}
    \rho_i(\beta) = \frac{\exp(\beta \log r(X_i))}{\frac{1}{n}\sum_{k=1}^n \exp(\beta \log r(X_k))},
\end{align}
interpolating between the empirical distribution ($\beta=0$) and the LLM density ($\beta=1$).

\paragraph{Heuristics.}
The mapping from semantic scores to sample weights is not unique; Statsformer supports a family of monotone transformations, including power-law and exponential weighting schemes, that can be applied to $\rho_i(\beta)$ to modulate the strength of the instance weights.
In practice, these choices can be fixed heuristically or selected via out-of-fold validation, ensuring that stronger weighting is applied only when it yields empirical gains and is otherwise automatically attenuated.

\subsection{Summary}
Both constructions anchor the monotone map in an optimization principle: MAP estimation for penalties, KL projection for sample weights. They remain compatible with the guardrail condition $\tau_0=1$, expose interpretable hyperparameters ($\alpha$ or $\beta$ as prior trust levels), and integrate seamlessly with the Statsformer stacking pipeline. To conclude:
\begin{itemize}
    \item The MAP-derived penalty family applies to any learner whose objective decomposes as empirical loss plus a sum of coordinate-wise penalties. Beyond the Lasso archetype, this includes GLMs, logistic/Poisson regression, kernel methods with diagonal priors, and neural networks equipped with group-wise $\ell_1/\ell_2$ regularization.
    \item The KL-projection weights require only that the base learner accept instance weights. Most practical models (e.g. generalized linear models, SVMs, tree ensembles, boosting algorithms, and many neural networks) already expose this interface.
    \item If a learner lacks the necessary hook (separable penalties or instance weights), Statsformer simply omits the corresponding prior family for that learner; the guardrail persists because the null path $(\alpha,\beta)=(0,0)$ always remains in the dictionary.
\end{itemize}

\paragraph{Remarks.}
\textit{
The aggregation guarantees that motivate the ``no-worse-than-null'' story belong to a classical convex aggregation lineage (e.g., Nemirovski, Yang, Lecué, Tsybakov). They only require that the dictionary contain a baseline element, so in the abstract Bayesian picture the prior can be any noisy but potentially informative source interpolated against a non-informative guardrail. The novelty of Statsformer lies in pairing that old backbone with foundation-model priors: we turn LLM-derived feature scores into the principled map families above, let the ensemble validate their usefulness, and interpret the theory as a safety certificate for data-driven LLM guidance.}

\section{Computational Analysis}\label{app:comp_analysis}
\subsection{Score Collection}
The score collection involves a one-time cost per dataset.
To reduce the LLM context length required for large feature sets, we divide the feature names into at most $\sqrt{p}$ batches and perform one API query per batch.
Up to several thousand features, the length of the prompt is dominated by content other than the list of features, so the number input tokens approximately scales with $\sqrt{p}$.
As the number of features grows further, the number of input tokens is $\mathrm{O}(p)$.

One score is output per feature, so the number of output tokens scales with $\mathrm{O}(p)$.

\subsection{Statistical Model}
Let $\mathcal{T}^{(T)}_m(n, p)$ be the time complexity of training base learner $b_m$ on a data matrix $\in \mathbb{R}^{n\times p}$ and $\mathcal{T}^{(I)}_m(p)$ be the time complexity of prediction for a $p$-dimensional datapoint.
Let $\mathcal{M}^{(T)}(n, p)$ be the time complexity of training the meta-learner.

\paragraph{Training.}
Training Statsformer consists of the following components: initial run of the base learners to generate out-of-fold predictions, training the meta-learner, and refitting the methods with non-zero coefficients.
The time complexity of the initial run is
\[\mathrm{O}\left(\sum_{m=1}^M k |\Theta_m|\mathcal{T}^{(T)}_m(n(k-1)/k, p)\right) ,\]
where $k$ is the number of CV folds.
In the worst case, we have to run every base learner-hyperparameter combination for each cross-validation fold.
As some base learners (e.g., Lasso) have internal cross-validation that automatically produces out-of-fold predictions, we may not need to explicitly run all base learners $k$ times for each $\theta \in \Theta_m$.

Training the meta-learner has complexity
\[\mathcal{M}^{(T)}\left(n, {\sum}_{m \in [M]} |\Theta_m|\right),\]
as the input to the meta-learner has one feature per model-hyperparameter combination.

Finally, the time complexity of retraining the base learners has complexity
\[\mathrm{O}\left(\sum_{m=1}^M |\Theta_m|\mathcal{T}^{(T)}_m(n, p)\right).\]
If the meta-learner assigns coefficients of $0$ to many base learners, then the retraining runtime will be lower by a factor proportional to the number of base learners with coefficient $0$.

Combining all steps, the training complexity of Statsformer is
\[\mathrm{O}\left(\sum_{m=1}^M k |\Theta_m|\mathcal{T}^{(T)}_m(n/k, p) + \mathcal{M}^{(T)}\left(n, {\sum}_{m \in [M]} |\Theta_m|\right)\right).\]

\paragraph{Inference.}
Performing inference on Statsformer involves performing inference on each base model-hyperparameter combination with a nonzero weight, resulting in a time complexity of
\[\mathrm{O}\left(\sum_{m=1}^M |\Theta_m|\mathcal{T}^{(I)}_m(p)\right).\]

\subsection{Controlling Computational Overhead}

Both the training and inference runtime of Statsformer is proportional to the number of models in the ensemble, or $\sum_m |\Theta_m|$, or $M$ times the average number of hyperparameters in $\Theta_m$.

In the case that $M$ and the average $\Theta_m$ are both relatively, we can reduce the computational cost via model preselection.
I.e., we can run each base model and use cross validation to choose the top $\hat{M} \ll M$ to add to the out-of-fold stacking model.

Perhaps more effective in reducing computational cost is simply reducing the size of $\Theta_m$.
In terms of model hyperparameters, the regularization strength for Lasso, e.g., is optimized separately from the out-of-fold stacking procedure (i.e., a regularization strength is fixed for each value of $\alpha$).
In addition, for models such as XGBoost, the hyperparameters can be optimized once for the base learner (with $\alpha=0$) and then fixed for the remainder of the configurations and out-of-fold stacking procedure.
In terms of $\alpha$ and $\beta$ sweeps, in practice we find that including $\approx 3$ values of $\alpha$ (and $\approx 2$-$3$ values of $\beta$, if sample weights are used) achieves sufficiently good accuracy.

To reduce the score querying cost, we can perform feature preselection: use a relatively good feature selector (e.g., mutual information) to select $\hat{p} < p$ features prior to running Statsformer.
The remaining features can be removed from the model, or assigned a default score.

\subsection{Empirical Compute and Cost}\label{appdx:empirical_compute}
Querying the LLM for feature scores is a one-time cost that is approximately linear in the number of features.
For the Lung Cancer dataset, which has 1,000 features and about 1,000 instances, the overall querying cost with the OpenAI \texttt{o3} model and five trials was \$2.80, with 50 cents for input tokens and \$2.30 for output tokens.
The overall querying time with 5 threads was 15 minutes, and could be further reduced with more parallelism.

For the smaller-dimensional German Credit dataset, which has 20 features and 1,000 instances, the overall querying cost was about 4 cents, with 1.5 cents for input tokens and 2.5 cents for output tokens.
The overall querying time was 40 seconds, without parallelism used.

All experiments were run on a shared server equipped with AMD EPYC 7713 64-core processors.
Although the CPU has 64 physical cores (128 threads), experiments were limited to 8 threads due to concurrent usage by other users.
See Table~\ref{tab:computation} for a comparison of our computational speed compared to baseline methods for three different problems.
By nature, Statsformer incurs an overhead over efficient statistical baselines like XGBoost and Lasso, due to the fact that we ensemble multiple methods.
However, we are consistently faster than AutoGluon, with a larger margin in higher dimensions.
The AutoML-Agent is much slower due to the interaction of LLM agents and the need for frequent retries upon agent failure.

\begin{table}[ht]
\centering
\scriptsize
\setlength{\tabcolsep}{4pt}

\begin{tabular}{@{}lcc|cccccccc@{}}
\toprule
& & &
\multicolumn{8}{c}{\textbf{Training Time}} \\
\cmidrule(l){4-11}
\textbf{Dataset} &
$n$ &
$p$ &
\textbf{XGBoost} &
\textbf{Lasso} &
\textbf{LightGBM} &
\textbf{LLM-} &
\textbf{Rand-} &
\textbf{Auto-} &
\textbf{AutoML-} &
\textbf{Stats-} \\
&
&
&
(Untuned) &
&
(Untuned) &
\textbf{Lasso} &
\textbf{Forests} &
\textbf{Gluon} &
\textbf{Agent} &
\textbf{former} \\
\midrule
\textbf{Bank} &
1000 & 16 &
2s &
12s & 1m 13s &
36s & 10s & 9m 42s & $\approx$5h & 8m49s \\

\textbf{ETP} &
189 & 1000 &
3s & 7s & 17s &
21s & 3s & 19m 2s & $\approx$5h & 2m 10s \\

\textbf{Lung} &
1017 & 1000 &
31s & 17s & 3m 47s &
45s & 20s & 1h 46m & $\approx$5h & 17m 50s \\
\bottomrule
\end{tabular}

\caption{\scriptsize Runtime of Statsformer (after score collection) and baselines for 80 different training and test splits (with training ratios ranging from 10\% to 80\%, and 10 different random seeds per training ratio), for three datasets of varying size.
For ETP, there were only 50 splits; due to the low-sample and imbalanced nature of the datasets, we only tested training ratios from 30\% to 70\%.
Note that the AutoML-Agent times are estimates due to the script occasionally hanging, requiring it to be killed and resumed.
As the time is dominated by agent interactions and not fitting of statistical models, this estimate is constant across datasets.} \label{tab:computation}
\end{table}

\section{Base Learner Curation}\label{appdx:base_learner}
We curate a collection of strong statistical base learners that naturally admit at least one of the three adapter mechanisms introduced in Section~\ref{sec:framework}: penalty-based, feature-importance-based, or instance-weighting-based prior injection. For each adapter, the corresponding prior-integrated objective is defined in Section~\ref{sec:framework}. Table~\ref{tab:base_learner} summarizes the base learners considered in this work together with the adapter(s) they support.

Importantly, the admissible adapters for a given base learner are determined by the structural properties of the learner itself. In particular, learners that support explicit regularization terms or feature-level coefficients naturally admit penalty-based or feature-importance-based adapters. For learners that do not allow such forms of modification, instance weighting remains universally applicable.

When a base learner admits multiple adapter mechanisms, we consider two complementary strategies. One may either include all corresponding prior-integrated variants as distinct candidates within the ensemble, or conduct an isolated base-learner study and select the best-performing adapter for that learner based on a held-out experimentation dataset.
In this section, we include plots from such experiments showing the improvement that scores can induce on each base learner.

\begin{table}[ht]
\centering
\caption{\small Base Learners Summary.} \label{tab:base_learner}
\scriptsize
\begin{tabular}{@{}llllllll@{}}
\toprule
\textbf{Learners} &
\textbf{Lasso} &
\textbf{XGBoost} &
\textbf{Random Forests} &
\textbf{Weighted Kernel SVM} &
\\
\midrule
\textbf{Adapters} & 
Penalty-weighting&
Feature-weighting & 
Sample-weighting, Feature-weighting
& Feature-weighting &
\\
\midrule
\textbf{References} & 
\citep{tibshirani1996lasso}&
\citep{chen2016xgboost} & 
\citep{breiman2001randomforest}
& \citep{Zhang2011WeightedSVM} &
\\
\bottomrule
\end{tabular}
\end{table}

In the following, we introduce each learner with its noted adapter in Table \ref{tab:base_learner}.

\paragraph{Penalty-weighted Lasso.}
We begin with the Lasso \citep{tibshirani1996lasso}, which serves as a canonical example of a linear model admitting penalty-based prior injection. Let $(x_i, y_i)_{i=1}^n$ denote the training data with $x_i \in \mathbb{R}^d$. The standard Lasso estimator solves
\[
\hat{\beta}
\;=\;
\arg\min_{\beta \in \mathbb{R}^d}
\frac{1}{n}\sum_{i=1}^n \ell(y_i, x_i^\top \beta)
\;+\;
\lambda \|\beta\|_1 ,
\]
where $\ell(\cdot,\cdot)$ denotes a convex loss and $\lambda>0$ is a regularization parameter.

To incorporate feature-level prior information, we adopt a penalty-weighted variant in which the $\ell_1$ penalty is reweighted according to a nonnegative importance vector $p \in \mathbb{R}^d_{\ge 0}$ derived from the prior. Specifically, we consider the objective
\[
\hat{\beta}_p
\;=\;
\arg\min_{\beta \in \mathbb{R}^d}
\frac{1}{n}\sum_{i=1}^n \ell(y_i, x_i^\top \beta)
\;+\;
\lambda \sum_{j=1}^d w_j |\beta_j|,
\]
where $w_j = g(p_j)$ for a monotone decreasing function $g(\cdot)$, so that features deemed more important by the prior are penalized less heavily. This formulation coincides with the adaptive Lasso when $w_j$ are data-driven, and reduces to the standard Lasso when all $w_j$ are equal.

\begin{figure}[htb]
    \centering
    \includegraphics[width=0.92\linewidth]{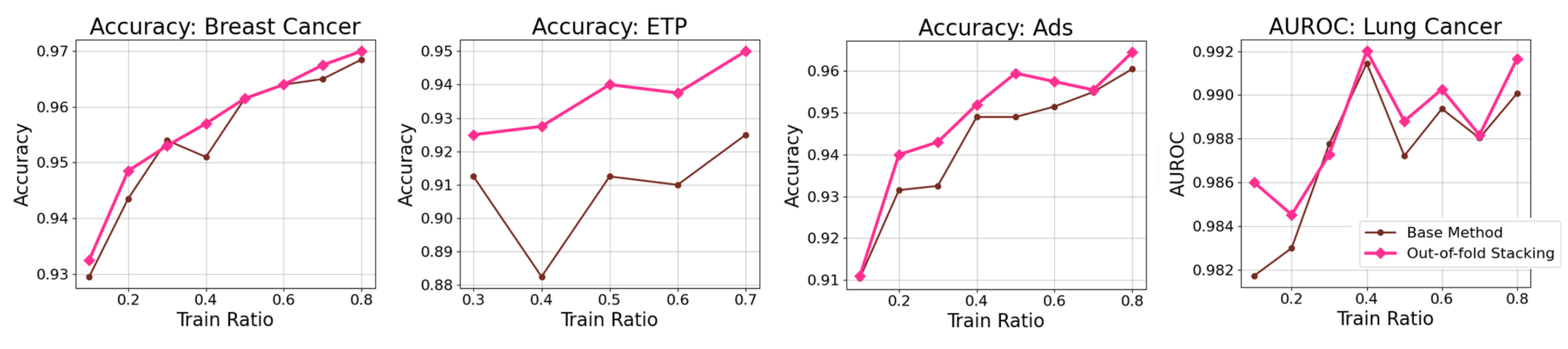}
    \caption{\scriptsize Single-learner study on selected datasets for prior injection into weighted Lasso (using the \texttt{adelie} Python package).}
    \label{fig:single-learner-lasso}
\end{figure}

Penalty-weighted Lasso naturally fits within our framework as a feature-level adapter, allowing the prior to modulate sparsity structure while preserving convexity and interpretability.
As a result, the model more reliably recovers prior-aligned sparse solutions, yielding improved performance (Figure~\ref{fig:single-learner-lasso}).

\paragraph{XGBoost.}
We consider gradient-boosted decision trees as implemented in XGBoost \citep{chen2016xgboost}, a strong nonlinear baseline widely used in tabular learning. While tree-based models do not admit explicit coefficient-level regularization in the same sense as linear models, they naturally allow feature-level prior injection through feature subsampling at decision tree nodes.

In particular, when subsampling is enabled, XGBoost samples a fixed proportion of features at each split.
Rather than sampling uniformly, we draw features according to a distribution proportional to the feature score $s_i(\alpha)$.

This modification biases split selection toward features assigned higher prior importance, while leaving the underlying boosting procedure unchanged. Feature reweighting inhection therefore serves as a simple and principled feature-level adapter compatible with tree-based learners.
Although boosting already incorporates iterative feedback, the prior-guided feature sampling further improves performance (Figure~\ref{fig:single-learner-xgboost}).

\begin{figure}[htb]
    \centering
    \includegraphics[width=0.92\linewidth]{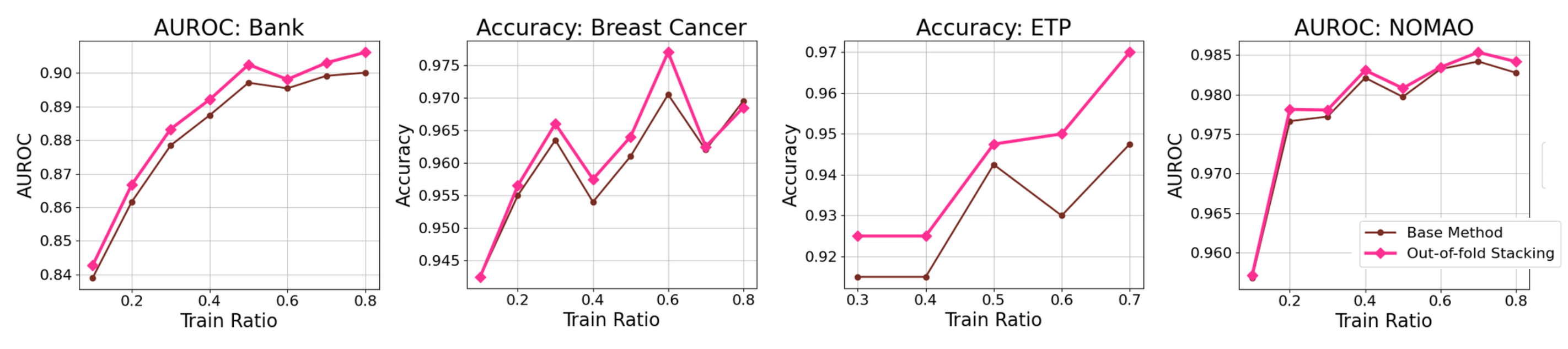}
    \caption{\scriptsize Single-learner study on selected datasets for prior injection into XGBoost (using the \texttt{xgboost} Python package). Feature weights control feature subsampling probabilities at decision tree nodes.}
    \label{fig:single-learner-xgboost}
\end{figure}

\paragraph{Random Forests.}
Random Forests \citep{breiman2001randomforest} are ensemble tree methods that do not naturally admit penalty-based regularization or continuous feature-level modification through their optimization objective.
However, like XGboost, they allow prior injection through the stochastic components of training, including instance weighting and feature subsampling.

Given nonnegative instance weights $\alpha_i$ derived from the prior, we train weighted Random Forests by modifying the bootstrap sampling distribution and split criteria to account for $\alpha_i$.
In practice, this is implemented by passing sample weights to the training procedure, so that observations deemed more informative by the prior exert greater influence on tree construction.

In addition, we incorporate feature-level priors by biasing the feature subsampling step at each split.
Specifically, features are oversampled with probabilities proportional to their prior importance scores, making semantically relevant features more likely to be considered during split selection while preserving the randomized structure of the forest.

Together, instance weighting and feature oversampling provide complementary mechanisms for prior injection in Random Forests when direct feature- or parameter-level regularization is unavailable.
See Appendix~\ref{appdx:prior-injection} for the formal construction of instance weights from feature priors, and Appendix~\ref{subsec:kl-instance} for their theoretical justification.
Empirically, we observe improved performance relative to standard Random Forests (Figure~\ref{fig:single-learner-rf}).

\begin{figure}[htb]
    \centering
    \includegraphics[width=0.92\linewidth]{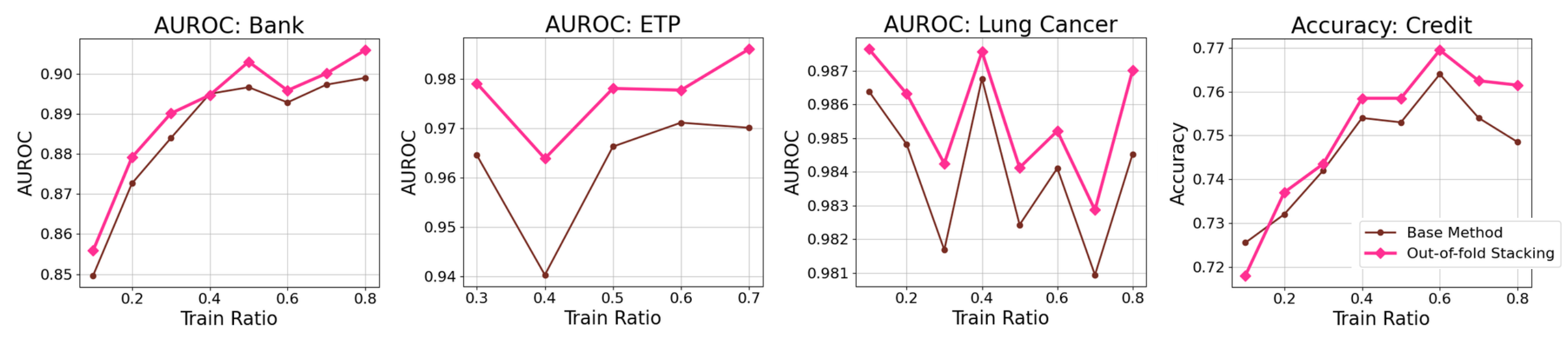}
    \caption{\scriptsize Single-learner study on selected datasets for prior injection into Random Forests, using the \texttt{scikit-learn} Python implementation.
    Instance weights are set according to Appendix~\ref{appdx:prior-injection}.
    Feature weights are introduced via oversampling features with replacement such that the feature space is doubled.}
    \label{fig:single-learner-rf}
\end{figure}

\paragraph{Weighted Kernel SVM.}
Support Vector Machines with nonlinear kernels provide another example of a learner admitting feature-level prior injection. Let $k(\cdot,\cdot)$ denote a positive-definite kernel. A weighted kernel SVM incorporates feature importance by modifying the kernel-induced metric.

Given a feature-importance vector $p \in \mathbb{R}^d_{\ge 0}$ and a monotone map $g(\cdot)$, we define a diagonal scaling matrix $D = \mathrm{diag}(g(p))$ and use the kernel
\[
k_p(x,z) = k(Dx, Dz).
\]
For example, under the Gaussian kernel, this corresponds to an anisotropic distance in feature space. In practice, this weighted kernel SVM can be implemented by scaling input features with $D$ and then applying a standard kernel SVM solver.

This construction allows the prior to shape the geometry of the reproducing kernel Hilbert space, biasing the classifier toward directions deemed important by the prior while retaining the convex optimization structure of SVM training.
As a result, the classifier prioritizes prior-aligned directions, leading to improved performance in Figure~\ref{fig:single-learner-kernel}.

\begin{figure}[htb]
    \centering
    \includegraphics[width=0.92\linewidth]{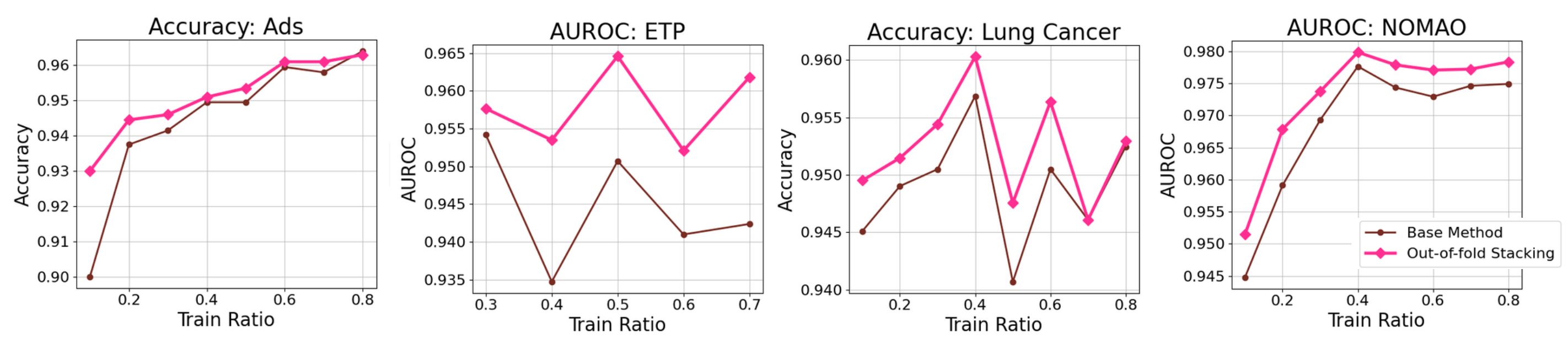}
    \caption{\scriptsize Single-learner study on selected datasets for prior injection into Kernel SVMs, using the \texttt{scikit-learn} implementation of Kernel SVMs under the radial basis function (RBF) kernel.
    Input features are scaled by $s_i(\alpha)$ before being passed into the SVM solver.}
    \label{fig:single-learner-kernel}
\end{figure}




\section{Deferred Experimental Details}\label{appdx:experimental-details}
\subsection{Data Splitting and Metrics}
To study performance as a function of training set size, we subsample the training data, sweeping the training ratio over ${0.1, 0.2, \ldots, 0.8}$ while fixing the test ratio at $0.2$. For each ratio, we generate 10 random splits using different random seeds (stratified for classification).
For classification datasets, we adjust the training-ratio sweep and test set size to ensure sufficient minority-class representation. Specifically, we require that each split contain at least four samples from each class and that the maximum training ratio be at most $1$ minus the test ratio.

We report accuracy and/or AUROC for classification tasks and MSE for regression, averaged across splits, as well as 95\% confidence intervals.

\subsection{Datasets}\label{appdx:datasets}
We evaluate our methods on a diverse collection of benchmark datasets spanning gene expression, marketing, finance, web, etc., ranging from $\approx\!\!10$ to $1000+$ features.

\paragraph{Breast Cancer.}
The GEMLeR repository \cite{stiglic2010gemler} provides a collection of gene expression datasets designed for benchmarking machine learning algorithms on microarray classification problems.
Specifically, we use the ``OVA Breast Cancer vs. Other'' dataset, which has 1545 samples and 10,936 features.
Each sample originates from the expO (Expression Project for Oncology) repository, which collects clinically annotated tumor tissue samples processed under standardized conditions using the same microarray platform, with associated clinical outcomes publicly available.

We map the original probe identifiers to gene symbols using the GPL570 annotation table from GEO, de-duplicate genes by averaging expression values across probes mapping to the same gene, and select the 1,000 genes with highest variance across samples to form the final feature set.

\paragraph{Bank Marketing.}
The UCI Bank Marketing dataset \cite{bank_marketing_222} records the outcomes of direct marketing campaigns (telephone calls) conducted by a Portuguese banking institution, where multiple contacts were often required to assess customer interest.
The task is binary classification: predicting whether a client subscribes to a term deposit based on demographic, financial, and campaign-related attributes.
The dataset comprises 45,211 instances with 16 features, including job type, marital status, education, and account balance.

Because semantic priors are more relevant in data-limited settings, we evaluate this dataset in a lower-sample regime.
Specifically, we use stratified sampling to subsample the dataset to 1,000 instances, preserving the class distribution.
This subsampling procedure is repeated with a different random seed for each train-test split evaluated.

\paragraph{ETP.}
Early T-cell precursor T-lineage acute lymphoblastic leukemia (ETP T-ALL) is a high-risk subtype of T-lineage acute lymphoblastic leukemia (T-ALL) that is challenging to diagnose due to its distinct immunophenotypic profile.
Using a dataset of 1000 gene expression levels from 189 T-ALL samples \cite{liu2017genomic}, we classify tumor samples into ETP T-ALL and non-ETP T-ALL.

\paragraph{German Credit.}
The UCI Statlog (German Credit) dataset \cite{uci_german_credit} is a benchmark dataset for classification of good or bad credit risk.
It contains 1,000 instances describing individuals applying for credit, with 20 features capturing demographic, financial, and credit-related information (e.g., age, credit history, employment status, and loan purpose). 

\paragraph{Internet Advertisements.}
The UCI Internet Ads dataset \cite{internet_advertisements_51} is a publicly available benchmark for binary classification, containing 3,279 instances with 1,555 features representing attributes of web pages (e.g., textual, visual, and structural characteristics). 
We perform mean imputation for missing values and use stratified sampling to subsample the dataset to 1,000 instances, preserving the class distribution as described for the Bank Marketing dataset.
The final task is to predict whether a web page contains an advertisement.

\paragraph{Lung Cancer.}
Data for the lung cancer dataset are obtained The Cancer Genome Atlas Program (TCGA) \cite{weinstein2013cancer}, a publicly available database of human tumors
Our sample consists of bulk RNA sequencing data from 516 samples from patients with lung adenocarcinoma (LUAD) and 501 samples from patients with lung squamous cell carcinoma (LUSC).
Only primary tumor samples are used.
The data are normalized, variance-stabilized, and transformed using DESeq2 \cite{love2014moderated}.
If more than one sequencing data is available for a patient, the average of the counts are taken.
Genes with fewer than 10 counts are filtered out.
We use the top 1000 most variable protein-coding genes in the downstream analyses.

\paragraph{Nomao.} 
The UCI Nomao dataset \cite{uci_nomao} is a benchmark for location de-duplication, where the task is to determine whether a pair of locations corresponds to the same place. 
It contains 34,465 instances, each representing a pair of locations, described by 120 features comparing attributes such as name, phone number, and geographic information. 
We use stratified sampling to subsample the dataset to 1,000 instances, preserving the class distribution as described for the Bank Marketing dataset.

\paragraph{Superconductivity.}
The UCI Superconductivity dataset \cite{uci_superconductivity} contains 21,263 instances of superconductors with 81 real-valued features extracted from chemical and structural properties. 
The prediction task is regression: estimating the critical temperature of each superconductor. 
We use subsampling to select 1,000 instances.

\subsection{Implementation}\label{appdx:implementation}
Statsformer is implemented in Python, and all code will be open-sourced on Github.
\paragraph{Base learners.}
The Random Forest and SVM base learners are implemented using \texttt{scikit-learn} \cite{scikit-learn}, XGBoost via its Python package, and Lasso via the \texttt{adelie} library for generalized linear models \cite{yang2024fastscalablepathwisesolvergroup}.
For Random Forests and XGBoost, we do not perform per-dataset hyperparameter tuning in order to limit overhead.
Since we outperform AutoGluon (Figure~\ref{fig:baseline-comparison}), which does perform hyperparameter tuning, this choice appears sufficient for the problems considered.
Details on each learner, including hyperparameters, are as follows:
\begin{itemize}
    \item \textit{Lasso}: \texttt{adelie} natively supports weighted Lasso, and we pass feature weights via the \texttt{penalty} argument of \texttt{grpnet}.
    The regularization strength $\lambda$ is selected prior to out-of-fold (OOF) stacking using 5-fold cross-validation and is held fixed during downstream ensembling.
    This procedure is adapted from \texttt{adelie}'s \texttt{cv\_grpnet}, with modifications to support stratified cross-validation and storage of OOF predictions.
    We sweep $\lambda$ from $\lambda_{\max}$ (the smallest value yielding an intercept-only model) to $10^{-2}\lambda_{\max}$ using 100 logarithmically spaced values.
    The value of $\lambda$ minimizing validation loss is selected and fixed for OOF stacking.

    \item \textit{Random Forests}: we set the number of trees to be 50, and otherwise use default parameters.
    Instance weights are passed directly into the \texttt{fit} or \texttt{fit} function as the \texttt{sample\_weight} argument.
    Features are oversampled with replacement by a factor of $1$ (resulting in $2p$ total columns after oversampling), with probabilities proportional to $s_i(\alpha)$. 

    \item \textit{XGBoost}: Feature weights are directly in the instantiation of a \texttt{DMatrix}, and are used for subsampling features at each decision tree node.
    To take advantage of feature weighting, we set \texttt{colsample\_bynode} to $\max(0.2, \min(1, 30/p))$ and \texttt{tree\_method} to ``hist''.
    We set the number of boosting rounds to 50 and otherwise use default parameters.

    \item \textit{Kernel SVM}: we use the Radial Basis Function (RBF) kernel and default hyperparameters.
    Feature importance is applied by first standardizing the data such that each feature has zero mean and unit variance, and then scaling each feature by score $s_i(\alpha)$ before calling \texttt{SVC.fit} or \texttt{SVR.fit}.
\end{itemize}

\paragraph{Meta-learner.}
Whereas our framework described in Section~\ref{sec:framework} uses a simplex-constrained meta-learner, in practice we find that non-negativity constraints are easier to optimize and achieve similar performance.

For regression, we use \texttt{scikit-learn}'s \texttt{ElasticNetCV} implementation with a non-negativity constraint. 
We employ 5-fold cross-validation to select the regularization strength, as automatically supported by \texttt{scikit-learn}, and otherwise follow its default settings.

For classification, \texttt{scikit-learn}'s logistic regression implementation does not natively support non-negativity constraints. 
To address this, we modify \texttt{LogisticRegressionCV} so that the logistic regression optimization is performed with non-negativity constraints using \texttt{scipy.optimize} and the L-BFGS-B solver. 
The cross-validation procedure for selecting regularization strength remains unchanged.

For multi-class problems, we adopt a one-versus-rest strategy, fitting a separate binary classifier for each class versus the remaining classes.

\paragraph{LLM Querying.}
All LLMs are accessed via web APIs: we use the OpenAI API for OpenAI models and OpenRouter for all others.
Each LLM is prompted to assign floating-point importance scores in the range $[0.1, 1]$ to features, reflecting their relevance to the classification or regression task.

When supported, we set the generation temperature to $0$ (greedy decoding).
Otherwise, to reduce stochasticity, we query each model for five independent trials and average the resulting feature scores.

We request structured outputs in JSON format of the form
\texttt{\{``scores'': \{``FEAT\_NAME\_01'': float, ``FEAT\_NAME\_02'': float, \dots\}\}},
where \texttt{FEAT\_NAME\_NN} corresponds to dataset-specific feature names.
Because not all models natively support structured outputs, we perform manual output validation as follows:
\begin{enumerate}
    \item Some models emit additional text alongside the JSON, so we extract all valid JSON objects from the model output.
    \item Each extracted JSON is validated using \texttt{Pydantic} to ensure it conforms to the schema
    \texttt{\{``scores'': dict[str, float]\}}.
    If validation fails, we proceed to the next extracted JSON.
    \item For validated outputs, all feature names are lowercased, and we explicitly check that the JSON contains the complete set of expected feature keys.
\end{enumerate}

To mitigate context length limitations and performance degradation for long prompts, features are queried in batches of $40$ by default.
If a querying or validation error occurs, we automatically retry the request up to five times per batch, prefixing the prompt with the retry index.
Batches within each trial are parallelized using a thread pool with five worker threads.

See Figures~\ref{fig:system_prompt_example} and \ref{fig:prompt_example} (in Appendix~\ref{sec:prompts}) for our system prompt and user prompt template, respectively.
For biological datasets, we use the OMIM RAG system developed in \cite{zhang2025llm-lasso}.

\subsection{Baseline Details}\label{appdx:baseline}
The statistical baselines used as Statsformer base learners (XGBoost, Random Forests, Lasso, Kernel SVMs) employ the same implementations and hyperparameter settings as in Statsformer (see Section~\ref{appdx:implementation}).
We also include LightGBM \cite{ke2017lightgbm}, using the \texttt{scikit-learn} implementation with default parameters.

For the LLM-Lasso baseline, we use our \texttt{adelie}-based Lasso, selecting the ``power of 1/importance'' transformation on the LLM scores via AUROC cross-validation \cite{zhang2025llm-lasso}.

In addition to classical statistical methods, we utilize two AutoML baselines, AutoGluon~\cite{erickson2020autogluontabularrobustaccurateautoml} and AutoLM-Agent~\cite{trirat2025automlagentmultiagentllmframework}. Details on these baselines are as follows:

\paragraph{AutoGluon.}
We use AutoGluon's \texttt{TabularPredictor} \cite{erickson2020autogluontabularrobustaccurateautoml} with a time limit of 120 seconds per dataset.
AutoGluon automatically performs hyperparameter tuning and model selection across multiple base learners (including neural networks, gradient boosting, and linear models) and ensembles them.
The problem type (binary classification, multiclass classification, or regression) is automatically inferred from the target variable.
AutoGluon's internal cross-validation and ensemble selection procedures are used without modification.

\paragraph{AutoML-Agent.} AutoML-Agent \cite{trirat2025automlagentmultiagentllmframework} is a multi-agent LLM framework that generates full-pipeline AutoML code.
It uses a \texttt{Data Agent} for preprocessing, a \texttt{Model Agent} for model selection and hyperparameter tuning, and an \texttt{Operation Agent} for code generation and execution.
LLM-based planning designs pipelines, which are then executed to produce trained models.
We run AutoML-Agent via the OpenAI interface with default settings, disabling web search and Kaggle API functions since the datasets are public.
The user prompt is modified to avoid manual train-test splitting (full prompt in Figure~\ref{fig:user_prompt_example_automl}, Appendix~\ref{sec:prompts}).
Performance is measured by evaluating the generated models on held-out test sets.

\subsection{Language Models Used}
We use OpenAI \texttt{o3} reasoning model for our main experiments. To demonstrate the robustness of our method across various LLM choices, we here present additional experimental results using different LLMs for score collection. For a comprehensive comparison, we experiment with both popular open-source frontier LLMs involving Qwen3 \citep{qwen3technicalreport}, DeekSeek-R1 \citep{deepseek_r1}, and proprietary LLMs involving Anthropic's Claude Opus 4.5 \citep{anthropic2025claudeopus}, OpenAI's \texttt{o1} \citep{openai_gpt_o1}, and Google's Gemini 2.5 \citep{google_gemini_2.5}. 

\section{Oracle Prior Simulation and Mechanistic Intuition}\label{appdx:sim}

In this section, we supplement results in Section~\ref{sec:experiment}
and provide more details on the oracle priors.
While Section~\ref{sec:experiment} demonstrates Statsformer's efficacy with current frontier LLMs, the framework’s performance is naturally limited by the quality of the semantic signal. To start, we decouple the architectural mechanism from prior quality by simulating an ``oracle prior.''
The goal of this simulation is to validate the Statsformer framework’s behavior with perfectly informative priors, providing a glimpse of its potential performance ceiling as LLM reasoning improves. 

We first provide mechanistic intuition for how highly informative priors can improve finite-sample behavior under monotone prior injection in Section \ref{appdx:oracle-gap}, and then validate these effects empirically through a controlled simulation with perfectly aligned priors in Section \ref{appdx:oracle-setup}-\ref{appdx:oracle-implication}.

\subsection{Mechanism Intuition of Highly Informative priors}\label{appdx:oracle-gap}
This section provides intuition for the types of performance improvements that
may be achievable when the external semantic prior is highly informative.
The goal is to clarify how monotone prior injection can improve the conditioning of
learning, e.g., by proposing a better regularization geometry, an improved input metric,
or a more favorable emphasis over samples, and to discuss which effects are plausible
within each adapter class.
A formal characterization of optimality is beyond the scope of this work; our purpose
is to sketch mechanisms and limitations in an accurate, model-agnostic way.

\paragraph{Setup.}
We call a feature-level prior $V=(v_1,\dots,v_p)$ \emph{highly informative} for a base
learner family if it is aligned with structural aspects of the target prediction
problem that materially affect finite-sample estimation for that family.
Depending on the learner, this alignment may correspond to (i) identifying relevant
coordinates or anisotropy for linear or GLM-type models, (ii) inducing an appropriate
input metric or margin geometry for distance-, kernel-, or margin-based methods, or
(iii) emphasizing samples that are especially informative under heteroskedasticity or
mild distributional shift.

In all cases, prior injection introduces no new labeled information; instead, it modifies
the empirical objective through a monotone transformation (e.g., rescaling or
reweighting), which can improve conditioning and finite-sample behavior without
expanding the underlying hypothesis class.

\subsubsection{Penalty-Based Adapters: Geometry Selection Within a Family}
For linear predictors and generalized linear models, penalty-based injection modifies
the regularization geometry through feature-specific weights.
When the weights align with the problem structure (e.g., approximate sparsity or
coordinate-wise signal strengths), weighted $\ell_1$-type penalties can improve
adaptivity and constants in oracle inequalities and, in some settings, support recovery
\citep{zou2006adaptive,bickel2009simultaneous}.
Importantly, such improvements remain constrained by the minimax limits of the
underlying statistical problem: weighting can help approach the performance of
well-tuned regularization within the same model family, but it cannot beat
information-theoretic limits for the target class.

From this perspective, Statsformer can be viewed as proposing a collection of
candidate regularization geometries via the prior and selecting among them via
validated aggregation, yielding gains when the prior is informative while retaining
prior-free baselines when it is not.

\subsubsection{Feature-Scaling Adapters: Metric Refinement}
Feature-scaling adapters inject priors by rescaling input coordinates before training.
For learners whose behavior depends on the input metric (e.g., kernel methods,
nearest-neighbor models, and margin-based classifiers), such rescaling can act as a
simple form of metric refinement.
When geometry is a dominant source of error, a well-aligned scaling can improve
conditioning and favorably affect complexity- or margin-based generalization
quantities \citep{bartlett2002rademacher,shawe2004kernel}.
These effects are typically most relevant in finite-sample regimes and should be
interpreted as improved conditioning rather than an expansion of model expressivity.

\subsubsection{Instance-Weight Adapters: Reweighting the Training Measure}
Instance-weight adapters modify the empirical risk by reweighting samples.
When weights emphasize observations that are more informative for the target risk
(e.g., under heteroskedastic noise or mild covariate shift), reweighting can improve
robustness and efficiency, affecting bias--variance tradeoffs and finite-sample
constants \citep{shimodaira2000improving,sugiyama2007covariate}.
For a fixed target distribution and hypothesis class, the main role of this weighting is to improve stability and estimation efficiency when uniform weighting is poorly matched to the data-generating process.


\paragraph{Implications for Statsformer.}
Highly informative priors can reduce effective statistical difficulty by improving
regularization geometry, metric conditioning, or sample emphasis. These benefits are
adapter-dependent and are limited to what can be achieved within the candidate family
of monotone transformations.
Statsformer is designed to recover such improvements when they are supported by data. In our oracle-prior simulations (Section Section \ref{appdx:oracle-setup}-\ref{appdx:oracle-implication} below), we observe large gains in regimes where the prior is highly accurate, highlighting meaningful headroom and motivating future work on improving the reliability of semantic priors.

\subsection{Setup and Experimental Details}\label{appdx:oracle-setup}
We define an \emph{oracle-informative} prior as a perfectly informative instance of a
\emph{highly informative} prior, aligned with the true data-generating structure.
In this simulation, oracle informativeness corresponds to exact identification of the
informative feature subset and their relative signal strengths for a linear
classification problem.

To this extent, we generate a series of random binary classification problems as follows.
First, fix training size $n$, dimension $p$, number of informative features $\hat{p}$, and target balance coefficient $c$ (e.g., $c=0.5$ means the dataset is balanced, and $c=0.1$ means 10\% of the samples are positive).

\begin{itemize}
    \item \textit{Data matrix}: Generate feature means $\mu_i \stackrel{\text{i.i.d.}}{\sim} \mathrm{Unif}([-10, 10])$ and standard deviations $\sigma_i \stackrel{\text{i.i.d.}}{\sim} \mathrm{Unif}([0.5, 5])$ for $i \in [p]$.
    Let $\mu$ and $\sigma^2$ be the corresponding $p$-dimensional vectors.
    Generate data matrix $X \in \mathbb{R}^{n\times p}$, with rows i.i.d. sampled from $x_j \sim \mathcal{N}(\mu, \mathrm{diag}(\sigma^2))$.

    \item \textit{Oracle prior} ($V^\ast$): Uniformly sample informative indices $I \subset [p]$ over all $\hat{p}$-dimensional sets of indices.
    For all $i \in I$, set $V_i^\ast \sim \mathrm{Unif}([-5, -0.5] \cup [0.5, 5]) + \mathcal{N}(0, 0.1)$.
    For all other indices, sample i.i.d. $\mathcal{N}(0, 0.1)$.

    \item \textit{Output classes}:  Sample signal $\hat{y} = \tanh(X) V$.
    Let $\hat{\sigma}(\hat{y})$ be the empirical standard deviation of $\hat{y}$ and $Q_c$ be its empirical $c$\textsuperscript{th} quantile.
    The output classes are then $y_j = \mathbf{1}\{\hat{y}_j +\mathcal{N}(0, 0.1 \hat{\sigma}(\hat{y})) > Q_c\}$.
\end{itemize}

We divide $\mathcal{D}(X, y)$ into train and test data via a 50-50 stratified split.
$V^\ast$ is passed into Statsformer as the feature prior.

We compare the performance of the Statsformer (Oracle) against the standard Stacking (No Prior) baseline on test set accuracy and AUROC, over four choices of dataset parameters:
\begin{itemize}
    \item $n=100, p=1000, \hat{p}=20, c=0.2$: this simulates a high-dimensional, underdetermined setting like the ETP dataset (where we see the highest performance gains over plain ensembling, as per Table~\ref{tab:statsformer_full_summary}).
    \item $n=300, p=500, \hat{p}=50, c=0.3$: slightly lower-dimensional, higher-data, less sparse, and more balanced.
    \item $n=50, p=500, \hat{p}=5, c=0.5$: very sparse and underdetermined.
    \item $n=500, p=30, \hat{p}=15, c=0.5$: a high-sample, low-feature regime where pure data-driven learning can suffice, decreasing the utility of prior information.
\end{itemize}
For each setting, we randomly sample $100$ datasets and compute the improvement of Statsformer over the no-prior ensemble (i.e., ``Statsformer (no prior)'' from Section~\ref{sec:experiment}), in terms of test AUROC and accuracy.

\subsection{Simulation Results}\label{appdx:oracle-results}
Figure~\ref{fig:oracle} presents histograms of the improvement defined above (Statsformer minus the no-prior stacking baseline) for each of the four simulation settings.
Table~\ref{tab:perfect_prior_featurewise} summarizes these distributions by reporting the mean improvement and corresponding 95\% confidence intervals.

In high-dimensional, underdetermined regimes (large $p$ relative to $n$), the gains are both consistent and substantial, with the improvement distribution exhibiting significant positive mass.
This behavior closely mirrors the regimes in which we observe the largest empirical gains, such as the ETP and Lung Cancer datasets in Table~\ref{tab:statsformer_full_summary} and Figure~\ref{fig:us_vs_stacking}.
In these settings, limited sample sizes make it difficult for purely data-driven methods to reliably identify the relevant structure, creating an opportunity for semantic priors to guide the base learners toward more informative regions of the hypothesis space.
The oracle-prior simulation confirms that when such guidance is accurate, Statsformer can consistently translate it into measurable performance gains.

By contrast, in the high-sample, low-dimensional setting ($n=500$, $p=30$), the average improvement is markedly smaller.
Here, the data alone is typically sufficient to estimate the underlying relationships, reducing the marginal utility of additional semantic information.
While Statsformer continues to outperform stacking on average, the gains are less uniform than in underdetermined regimes, illustrating the framework’s ability to appropriately attenuate the influence of priors when they provide limited additional signal.

\begin{table}[t]
\centering
\small
\begin{tabular}{lcc}
\toprule
\textbf{Setting} $(n, p, \hat{p}, c)$ & \textbf{AUROC improvement} & \textbf{Accuracy improvement} \\
\midrule
$(100, 1000, 20, 0.2)$ & $0.1473 \pm 0.0200$ & $0.0464 \pm 0.0079$ \\
$(300, 500, 50, 0.3)$  & $0.0629 \pm 0.0066$ & $0.0544 \pm 0.0061$ \\
$(50, 500, 5, 0.5)$   & $0.0910 \pm 0.0217$ & $0.0764 \pm 0.0203$ \\
$(500, 30, 15, 0.5)$  & $0.0014 \pm 0.0005$ & $0.0014 \pm 0.0018$ \\
\bottomrule
\end{tabular}
\caption{\scriptsize Performance improvements under oracle semantic priors.
We report mean improvement over the no-prior (stacking) baseline, with $\pm$95\% confidence intervals, for both AUROC and accuracy.
Note that these are raw metric differences and not the percentage metrics presented in Table~\ref{tab:statsformer_full_summary}.} \label{tab:perfect_prior_featurewise}
\end{table}

\begin{figure}[htbp]
    \centering
    \includegraphics[width=0.95\linewidth]{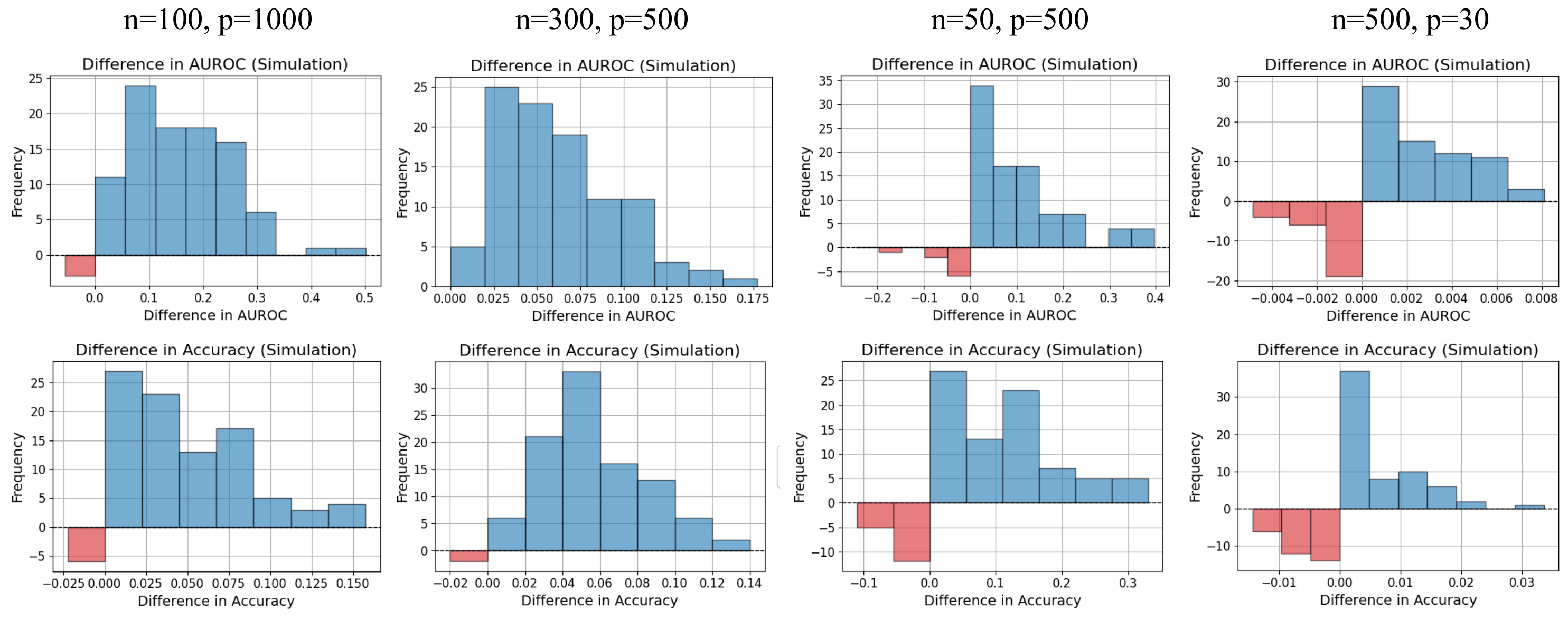}
    \caption{\scriptsize Results from the oracle prior simulation: histograms of Statsformer test $\{\text{AUROC}, \text{accuracy}\}$ improvements, computed as the metric difference (Statsformer minus no-prior stacking).
Each histogram is computed over 100 randomly generated datasets for each of the four regimes described above.
    }
    \label{fig:oracle}
\end{figure}

\subsection{Implications}\label{appdx:oracle-implication}
This simulation serves two key purposes.
First, it validates the underlying mechanism: the monotone adapters (Section~\ref{sec:framework}) are capable of exploiting high-quality semantic signals when they are available.
This is evidenced by the substantial gains over the no-prior stacking baseline, particularly in underdetermined regimes.
Second, it highlights the performance headroom of the framework: by contrasting the oracle results (Appendix Figure~\ref{fig:oracle}) with those obtained using current frontier LLMs (Figure~\ref{fig:baseline-comparison} in the main text), we illustrate the potential for further gains as the quality of semantic priors improves.



\section{Deferred Experimental Results}
\subsection{Deferred Comparisons with Baselines}
Figure~\ref{fig:appdx_us_vs_baselines} shows Statsformer compared to baselines for metrics not included in Figure~\ref{fig:baseline-comparison} (we compute both accuracy and AUROC for classification but display only one per dataset in the main figure). 
The results are consistent with the main paper: Statsformer is consistently the top performer or among the top two methods.

\begin{figure}[htbp!]
    \centering
    \includegraphics[width=1.0\linewidth]{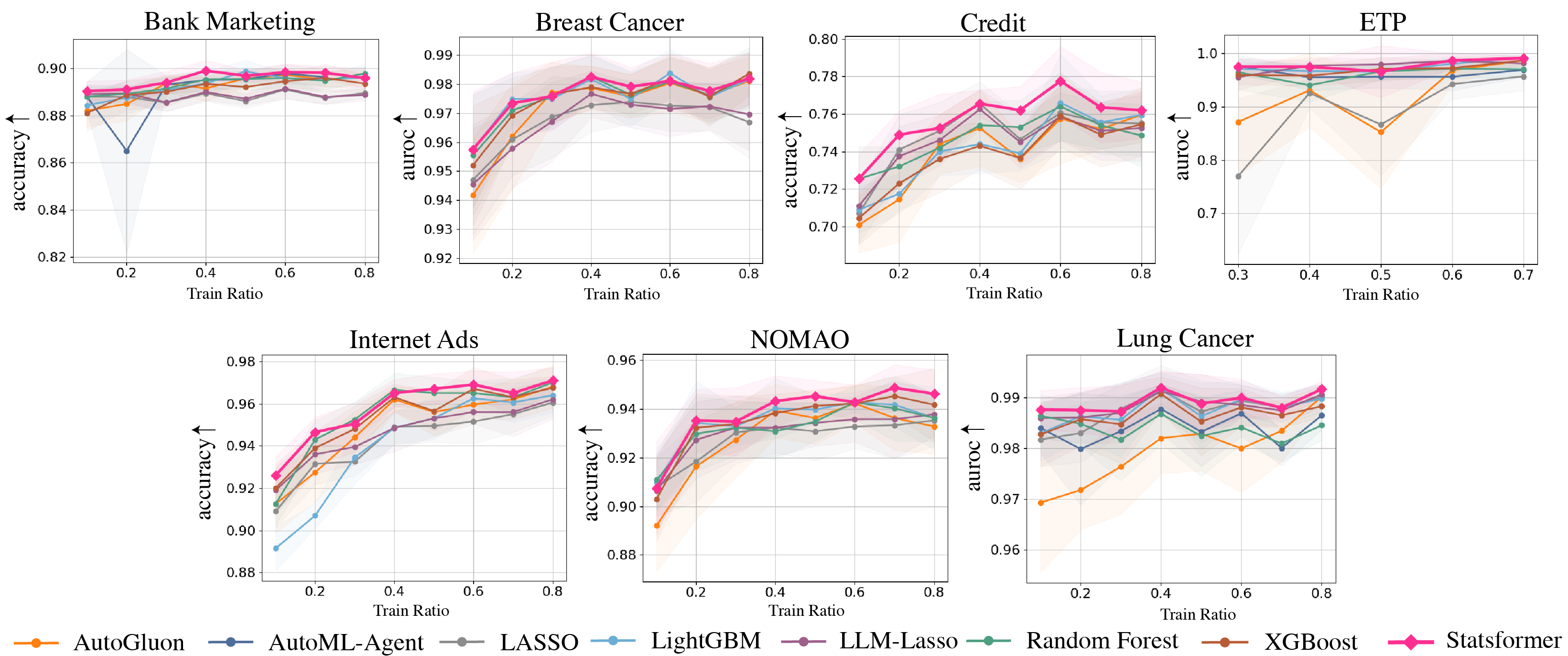}
    \caption{\scriptsize Comparison of Statsformer to baselines for the metrics not provided in Figure~\ref{fig:baseline-comparison} in the main paper.
    As in Figure~\ref{fig:baseline-comparison}, we plot the mean of the selected metrics 10 different train-test splits (selected via stratified splitting), as well as 95\% confidence intervals.
    The AutoML-Agent comparison is only included in three of the datasets due to computational constraints.} \label{fig:appdx_us_vs_baselines}
\end{figure}

\subsection{Deferred Comparison with Statsformer (No Prior)}

Figure~\ref{fig:appdx_us_vs_stacking} compares Statsformer to the no-prior variant (plain out-of-fold stacking) for datasets not shown in Figure~\ref{fig:us_vs_stacking}. 
While the gains are modest, they are consistent.
This consistent improvement is also reflected in the percentage improvements and win ratios reported in Tables~\ref{tab:statsformer_full_summary} of the main paper.

\begin{figure}[htbp!]
    \centering
    \includegraphics[width=1.0\linewidth]{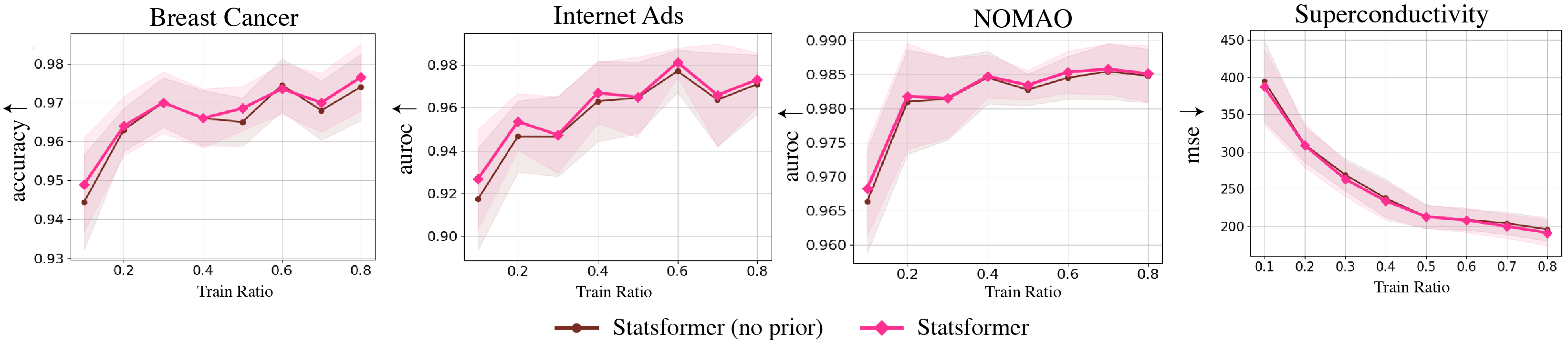}
    \caption{\scriptsize Comparison of Statsformer to the no-prior version of Statsformer, for datasets not shown in the main paper (Figure~\ref{fig:us_vs_stacking}).
    Due to less informative priors, the gains are smaller than for the datasets selected in Figure~\ref{fig:us_vs_stacking}, but they are nonetheless consistent.}
    \label{fig:appdx_us_vs_stacking}
\end{figure}
\begin{figure}[htbp!]
    \centering 
    \includegraphics[width=1.0\linewidth]{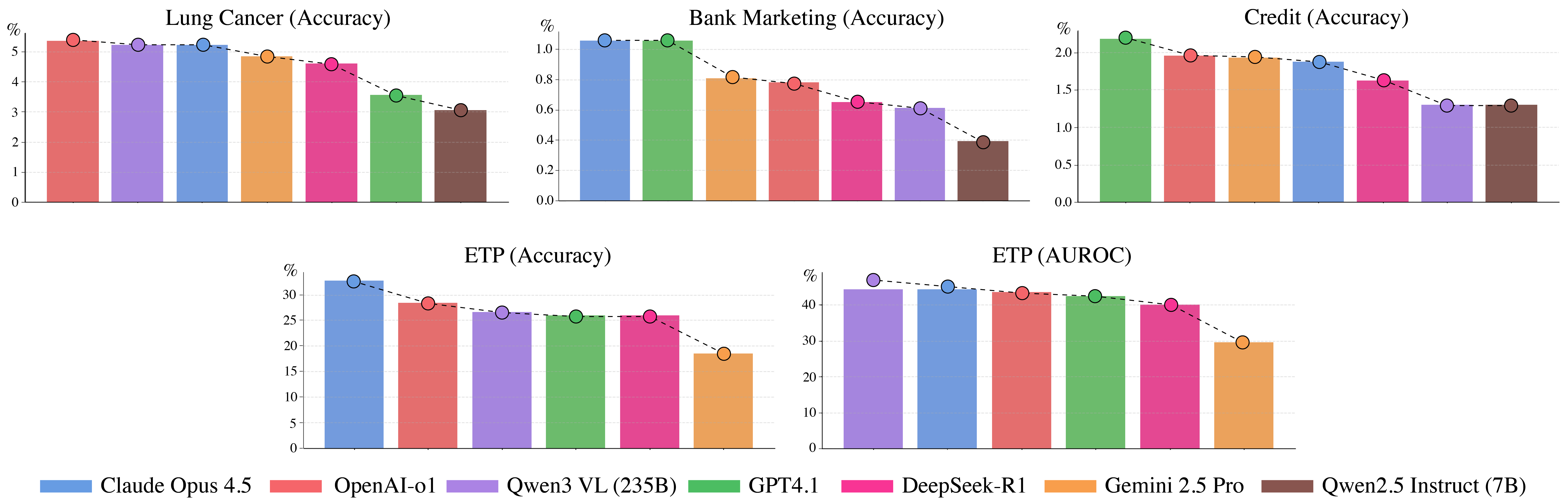}
    \caption{\scriptsize Additional experimental results for LLM model ablation, plotting mean percentage improvement.
    The improvement is calculated as the mean metric difference of Statsformer and the no-prior version, divided by the mean baseline error (i.e., Error or $1-\text{AUROC}$), and expressed as a percentage.
    As Qwen2.5 Instruct (7b) failed to produce results for the ETP dataset (due to difficulties parsing some of the gene names), it has been omitted from the ETP comparison.}
    \label{fig:model_ablation_append}
\end{figure}

\subsection{Deferred LLM Ablations}\label{model_ablation_append}

For model ablation study, we access all LLMs via either \texttt{OpenAI} API or \texttt{OpenRouter}. For non-OpenAI models, we consider Qwen series involving Qwen2.5 Instruct (7B) and Qwen3 VL (235B).
We consider Claude Opus 4.5, DeepSeek-R1, and Gemini 2.5 Pro.
We compute exactly the same statistics as in Figure~\ref{fig:model_ablation_main}.
Figure \ref{fig:model_ablation_append} presents additional experimental results, which is consistent with our observation discussed in main Section \ref{result_main}.
Note that for the ETP dataset, we do not report Qwen2.5 Instruct (7B) results, as the model consistently failed to produce valid scores (i.e., it was unable to parse some of the feature names).

\section{Prompts}\label{sec:prompts}

In Figures~\ref{fig:system_prompt_example}, \ref{fig:prompt_example}, and \ref{fig:task_description_example}, we show system and user prompts used in producing LLM-generated scores for Statsformer.
Figure~\ref{fig:user_prompt_example_automl} shows an example prompt for the AutoML-Agent baseline.

\begin{figure}[p]
    \centering
    \begin{tcolorbox}[
        width=0.92\linewidth,
        colback=gray!5,
        colframe=gray!50,
        boxrule=0.5pt,
        arc=2pt,
        left=6pt,
        right=6pt,
        top=6pt,
        bottom=6pt
    ]
    \scriptsize 
    You are an expert statistical reasoning assistant trained in both machine learning and scientific literature interpretation. Your task is to estimate feature importance for predictive modeling problems in high-dimensional, low-sample-size settings (such as biomedical prediction, genomics, or other sparse data domains). \\

    Your goals are: \\
    - **Scientific caution**: Prioritize well-established knowledge and plausible domain-specific/statistical rationale. Do not invent evidence or cite specific studies unless they are well-known and generalizable. \\
    - **Analytical rigor**: Reason about each feature in the context of the task using mechanistic, statistical, or empirical justification. Avoid overconfidence. \\
    - **Score calibration**: Assign scores conservatively. It is acceptable (and often desirable) that many features receive low scores if evidence of importance is weak or unclear. \\
    - **Faithful formatting**: Output strictly valid JSON in the exact format requested by the user prompt. \\
    
    Avoid: \\
    - Speculation or fabricated evidence. \\
    - Extraneous commentary or explanations outside of the JSON output. \\
    - Mentioning uncertainty explicitly in text (reflect uncertainty through the magnitude of scores). \\
    
    Your goal is to produce reasoned, evidence-based feature importance scores that can serve as priors for statistical modeling.
    \end{tcolorbox}
    \caption{\scriptsize System prompt used for querying feature-importance scores from LLMs. }
    \label{fig:system_prompt_example}
\end{figure}

\begin{figure}[htbp]
    \centering
    \begin{tcolorbox}[
        width=0.92\linewidth,
        colback=gray!5,
        colframe=gray!50,
        boxrule=0.5pt,
        arc=2pt,
        left=6pt,
        right=6pt,
        top=6pt,
        bottom=6pt
    ]
    \scriptsize 
    **Context**: \{\{context\}\}

    **Prediction Task**: \{\{task\}\} \\
    
    You are asked to assign importance scores to a set of features for use in a statistical prediction model (e.g., Lasso, XGBoost, or logistic regression). The data are high-dimensional with limited samples, so parsimony and caution are critical. \\
    
    **Objective**: \\
    For each feature in the provided list, assign an importance score between 0.1 and 1.0 (inclusive).  \\
    - A score closer to 1.0 indicates strong, well-established relevance or a robust mechanistic rationale for predicting "\{\{task\}\}". \\
    - A score near 0.1 indicates weak, uncertain, or unsupported relevance. \\
    - Most features may appropriately receive low scores. \\
    
    **Reasoning Guidelines**: \\
    1. Base your assessment on established knowledge, logical domain reasoning, or widely accepted statistical principles.  \\
    2. Avoid speculation or over-interpretation. \\
    3. You may reason internally but must output only the final scores. \\
    4. Do not skip any features. \\
    
    **Output Requirements**: \\
    - Output strictly valid JSON and nothing else.  \\
    - Use the format provided below. \\
        
    **Output Format**: \\
    \{"scores": \{ \\
            "FEATURE\_NAME\_01": floating\_point\_score\_value, \\
            "FEATURE\_NAME\_02": floating\_point\_score\_value, \\
            ...one score per feature name. \\
        \}\} \\
    
    **Features**:
    \{\{features\}\}
    \end{tcolorbox}
    \caption{\scriptsize User prompt format used to elicit feature-importance scores from LLMs. \texttt{Context} and \texttt{Task} are specified for each task; see Figure~\ref{fig:task_description_example} for an example, and refer to our codebase for the remainder of the datasets.
    \texttt{Features} is replaced with the list of feature names in the current batch, with Python array syntax.}
    \label{fig:prompt_example}
\end{figure}

\begin{figure}[htbp]
    \centering
    \begin{tcolorbox}[
        width=0.92\linewidth,
        colback=gray!5,
        colframe=gray!50,
        boxrule=0.5pt,
        arc=2pt,
        left=6pt,
        right=6pt,
        top=6pt,
        bottom=6pt
    ]
    \scriptsize 
    \textbf{Context}: We have bulk RNA sequencing data (microarray analysis) derived from a clinically annotated set of deidentified tumor samples. Each feature is the expression level for a given gene. We wish to build a statistical model that classifies samples into the categories "Early T-cell precursor Acute Lymphoblastic Leukemia/Lymphoma (ETP)" and "non-ETP".

    \textbf{Task}: classifying ETP (early T-cell precursor Acute Lymphoblastic Leukemia/Lymphoma) vs. non-ETP.
    \end{tcolorbox}
    \caption{\scriptsize Example task description for the ETP dataset. }
    \label{fig:task_description_example}
\end{figure}

\begin{figure}[htbp]
    \centering
    \begin{tcolorbox}[
        width=0.92\linewidth,
        colback=gray!5,
        colframe=gray!50,
        boxrule=0.5pt,
        arc=2pt,
        left=6pt,
        right=6pt,
        top=6pt,
        bottom=6pt
    ]
    \scriptsize
    Build a binary classification, multiclass classification, regression model to predict the 'target' column from the provided dataset. The dataset is located at the data path provided. \\

    CRITICAL REQUIREMENTS - FOLLOW EXACTLY: \\
    1. The dataset provided contains ONLY training data. DO NOT perform any train/test splits. \\
    2. DO NOT import or use \texttt{train\_test\_split()} from \texttt{sklearn.model\_selection} -- it is FORBIDDEN. \\
    3. Use ALL rows in the CSV file for training -- load the entire dataset and train on it. \\
    4. The data will have columns named 'feature\_0', 'feature\_1', etc. and a 'target' column. \\
    5. You may create a preprocessing pipeline (e.g., \texttt{ColumnTransformer}) for handling missing values and encoding. \\
    6. DO NOT use feature selection techniques that remove features (like RFE) -- keep all original features. \\
    7. DO NOT use resampling techniques (like SMOTE) -- use the data as-is. \\
    8. Train your model on the preprocessed data. \\
    9. Save BOTH the model AND the preprocessor together to 'model.pkl' using \texttt{joblib.dump()}. \\
    10. Save them as a dictionary: \texttt{\{'model': trained\_model, 'preprocessor': preprocessor\}}. \\
    11. The preprocessor must be able to transform new data with the same column structure (feature\_0, feature\_1, ..., target). \\
    12. The model must have \texttt{.predict()} and \texttt{.predict\_proba()} methods for classification (or \texttt{.predict()} for regression). \\
    13. Do not create validation or test sets -- only train on the full dataset. \\
    14. Do not evaluate the model in the code -- just train and save it. \\

    EXAMPLE CODE STRUCTURE: \\
    \texttt{```python} \\
    \texttt{import pandas as pd} \\
    \texttt{import joblib} \\
    \texttt{from sklearn.preprocessing import StandardScaler, OneHotEncoder} \\
    \texttt{from sklearn.compose import ColumnTransformer} \\
    \texttt{from sklearn.ensemble import RandomForestClassifier} \\
    \texttt{~} \\
    \texttt{\# Load ALL data (no splitting)} \\
    \texttt{df = pd.read\_csv('data\_path/data.csv')} \\
    \texttt{X = df.drop('target', axis=1)} \\
    \texttt{y = df['target']} \\
    \texttt{~} \\
    \texttt{\# Create preprocessor} \\
    \texttt{preprocessor = ColumnTransformer([...])} \\
    \texttt{X\_processed = preprocessor.fit\_transform(X)} \\
    \texttt{~} \\
    \texttt{\# Train model on ALL data} \\
    \texttt{model = RandomForestClassifier(...)} \\
    \texttt{model.fit(X\_processed, y)} \\
    \texttt{~} \\
    \texttt{\# Save model AND preprocessor together} \\
    \texttt{joblib.dump(\{'model': model, 'preprocessor': preprocessor\}, 'model.pkl')} \\
    \texttt{```} \\

    Remember: NO \texttt{train\_test\_split}, NO feature selection that removes features, NO resampling. Train on ALL data and save model + preprocessor together.
    \end{tcolorbox}
    \caption{\scriptsize Dynamic user prompt generated for the coding agent.}
    \label{fig:user_prompt_example_automl}
\end{figure}

\section{Societal Impact}\label{sec:broader-impact}
We note that the use of large language models to generate semantic priors may reflect biases present in their training data, and the use of closed-source models can raise transparency and reproducibility considerations.


\end{document}